\documentclass[sigconf]{acmart}

\usepackage{amsthm}
\theoremstyle{definition}

\usepackage{amssymb,amsthm,amsfonts,amsmath}
\newtheorem{theorem}{Theorem}
\newtheorem{lemma}{Lemma}
\newtheorem{proposition}{Proposition}
\usepackage{graphicx}
\usepackage{subcaption}
\usepackage{algorithm}
\usepackage{algcompatible}
\usepackage{xcolor}

\newcommand{\E}{\mathcal{E}}
\newcommand{\GG}{\mathcal{G}}
\newcommand{\T}{\mathcal{T}}
\newcommand{\D}{\mathcal{D}}
\newcommand{\X}{\mathcal{X}}
\newcommand{\Y}{\mathcal{Y}}

\newcommand{\V}{\mathcal{V}}

\newcommand{\R}{\mathbb{R}}

\newcommand{\1}{\mathbf{1}}

\newcommand{\bu}{\boldsymbol{u}}
\newcommand{\bv}{\boldsymbol{v}}
\newcommand{\bt}{\boldsymbol{t}}
\newcommand{\bw}{\boldsymbol{w}}

\newcommand{\tk}{\tilde{k}}

\newcommand{\softmin}{\mathit{softmin}}
\newcommand{\softmax}{\mathit{softmax}}
\newcommand{\EMD}{\text{EMD}}

\def\BibTeX{{\rm B\kern-.05em{\sc i\kern-.025em b}\kern-.08emT\kern-.1667em\lower.7ex\hbox{E}\kern-.125emX}}
    
\begin{document}

%
% The "title" command has an optional parameter, allowing the author to define a "short title" to be used in page headers.
\title[Scalable Global Alignment Graph Kernel Using Random Features]{Scalable Global Alignment Graph Kernel Using Random Features: From Node Embedding to Graph Embedding}
%
% The "author" command and its associated commands are used to define the authors and their affiliations.
% Of note is the shared affiliation of the first two authors, and the "authornote" and "authornotemark" commands
% used to denote shared contribution to the research.

% \author[$\star\thanks{Both authors contributed equally to this work}$]{Qi Lei}
% \author[$\dagger^\ast$]{Lingfei Wu}
% \author[$\dagger$]{Pin-Yu Chen}
% \author[$\star$]{Alexandros G. Dimakis}
% \author[$\star\ddagger$]{Inderjit S. Dhillon}
% \author[$\dagger$]{Michael Witbrock}
% \affil[ ]{$^\star$ UT Austin \hspace{1cm} $^\dagger$ IBM Research \hspace{1cm} $^\ddagger$ Amazon}
% %\affil[$\star$]{UT Austin}
% %\affil[$\dagger$]{IBM Research}
% %\affil[$\ddagger$]{Amazon/A9}
% \affil[ ]{\texttt{\{leiqi@ices, dimakis@austin, dhillon@cs\}.utexas.edu} }
% \affil[ ]{\texttt{\{wuli@us., pin-yu.chen@, witbrock@us.\}ibm.com }}

\author{Lingfei Wu$^\star$, Ian En-Hsu Yen$^\ast$, Zhen Zhang$^\dagger$, Kun Xu$^\star$, Liang Zhao$^+$, Xi Peng$^\circ$, Yinglong Xia$^\bullet$, Charu Aggarwal$^\star$ }\authornote{Corresponding author: Lingfei Wu. Email: wuli@us.ibm.com}
\affiliation{\institution{IBM Research$^\star$, CMU$^\ast$, WUSTL$^\dagger$, GMU$^+$, UDEL$^\circ$, Huawei$^\bullet$}}

% \author{Lingfei Wu}
% %\authornote{Dr.~Trovato insisted his name be first.} \authornote{Corresponding author: wuli@us.ibm.com}
% \affiliation{\institution{IBM Research}}
% \email{wuli@us.ibm.com}

% \author{ Ian En-Hsu Yen}
% \affiliation{\institution{Carnegie Mellon University}}
% \email{eyan@cs.cmu.edu}

% \author{Zhen Zhang }
% \affiliation{\institution{Washington University in St. Louis}}
% \email{zhen.zhang@wustl.edu}

% \author{Kun Xu }
% \affiliation{\institution{IBM Research}}
% \email{xukun@pku.edu.cn}

% \author{Liang Zhao }
% \affiliation{\institution{George Mason University}}
% \email{lzhao9@gmu.edu}

% \author{Xi Peng}
% \affiliation{\institution{University of Delaware}}
% \email{xpeng@udel.edu}

% \author{Yinglong Xia}
% \affiliation{\institution{Huawei Research}}
% \email{yinglong.xia.2010@ieee.org}

% \author{Charu Aggarwal}
% %\authornote{Dr.~Trovato insisted his name be first.}
% \affiliation{\institution{IBM Research}}
% \email{charu@us.ibm.com}

\renewcommand{\shortauthors}{Lingfei Wu, Ian En-Hsu Yen, Zhen Zhang, et al.}
\fancyhead{}
%
% The abstract is a short summary of the work to be presented in the article.
\begin{abstract}
 Graph kernels are widely used for measuring the similarity between graphs.
 Many existing graph kernels, which focus on local patterns within graphs rather than their global properties, suffer from significant structure information loss when representing graphs. Some recent global graph kernels, which utilizes the alignment of geometric node embeddings of graphs,  yield state-of-the-art performance. However, these graph kernels are not necessarily positive-definite. More importantly, computing the graph kernel matrix will have at least quadratic {time} complexity  in terms of the number and the size of the graphs. 
 In this paper, we propose a new family of global alignment graph kernels, which take into account the global properties of graphs by using geometric node embeddings and an associated node transportation based on earth mover's distance. Compared to existing global kernels, the proposed kernel is positive-definite. Our graph kernel is obtained by defining a distribution over \emph{random graphs}, which can naturally yield random feature approximations. The random feature approximations lead to our graph embeddings, which is named as ``random graph embeddings" (RGE).
 In particular, RGE is shown to achieve \emph{(quasi-)linear scalability} with respect to the number and the size of the graphs. 
 The experimental results on nine benchmark datasets demonstrate that RGE outperforms or matches twelve state-of-the-art graph classification algorithms.

 \end{abstract}




%
% This command processes the author and affiliation and title information and builds
% the first part of the formatted document.
\copyrightyear{2019} 
\acmYear{2019} 
\setcopyright{acmlicensed}
\acmConference[KDD '19]{The 25th ACM SIGKDD Conference on Knowledge Discovery and Data Mining}{August 4--8, 2019}{Anchorage, AK, USA}
\acmBooktitle{The 25th ACM SIGKDD Conference on Knowledge Discovery and Data Mining (KDD '19), August 4--8, 2019, Anchorage, AK, USA}
\acmPrice{15.00}
\acmDOI{10.1145/3292500.3330918}
\acmISBN{978-1-4503-6201-6/19/08}
\maketitle

%%%%%%%%%%%%%%%%%%%%%%%%%%%%%%%%%%%%%%%%%%%%%%%%%%%%%%%%%%%%%%%%%%%%%%%%%%%%%%%
\section{Introduction}
Graph kernels are one of the most important methods for graph data analysis and have been successfully applied in diverse fields such as 
disease and brain analysis \cite{mokhtari2013decoding,chen2017revisiting}, 
chemical analysis \cite{ralaivola2005graph}, 
image action recognition and scene modeling \cite{wang2013directed,fisher2011characterizing}, 
and malware analysis\cite{wagner2009malware}.
%and biomedical text categorization \cite{bleik2013text}. 
Since there are no explicit features in graphs, a kernel function corresponding to a high-dimensional feature space provides a flexible way to represent each graph and to compute similarities between them. Hence, much effort has been devoted to designing feature spaces or kernel functions for capturing similarities between structural properties of graphs.% \cite{shervashidze2011weisfeiler}. 

The first line of research focuses on local patterns within graphs \cite{gartner2003graph,shervashidze2009fast}. Specifically, these kernels recursively decompose the graphs into small sub-structures, and then define a feature map over these sub-structures for the resulting graph kernel. Conceptually, these notable graph kernels can be viewed as instances of a general kernel-learning framework called R-convolution for discrete objects \cite{haussler1999convolution,shervashidze2011weisfeiler}.  However, the aforementioned approaches consider only local patterns rather than global properties, which may substantially limit effectiveness in some applications, depending on the underlying structure of graphs. Equally importantly, most of these graph kernels scale poorly to large graphs due to their at least quadratic time complexity in terms of the number of graphs and cubic time complexity in terms of the size of graphs. 

Another family of methods use geometric embeddings of graph nodes to capture global properties, which has shown great promise, achieving state-of-the-art performance in graph classification \cite{johansson2014global,johansson2015learning,nikolentzos2017matching}. However, these  global graph kernels based on matching node embeddings between graphs may suffer from the loss of positive definiteness. Furthermore, the majority of these approaches have at least quadratic complexity in terms of either the number of graph samples or the size of the graph. 

To address these limitations of existing graph kernels, we propose a new family of global graph kernels that take into account the global properties of graphs, based on recent advances in the distance kernel learning framework \cite{wu2018d2ke}. The proposed kernels are truly \emph{positive-definite (p.d.)} kernels constructed from a random feature map given by a transportation distance %\cite{solomon2016continuous} 
between a set of geometric node embeddings of raw graphs and those of random graphs sampled from a distribution. In particular, we make full use of the well-known \emph{Earth Mover's Distance (EMD)}, computing the minimum cost of transporting a set of node embeddings of raw graphs to the ones of random graphs. To yield an efficient computation of the kernel, we derive a \emph{Random Features (RF)} approximation using a limited number of random graphs drawn from either data-independent or data-dependent distributions. The methods used to generate high-quality random graphs have a significant impact on graph learning. We propose two different sampling strategies depending on whether we use node label information or not. 
%This technique, \emph{Random Graph Embeddings (RGE)} from the RF approximation, is extremely simple to implement and compute, and is also fully parallelizable and highly extendable. 
Furthermore, we note that each building block in this paper - geometric node embeddings and EMD - can be replaced by other node embeddings methods \cite{johansson2015learning,zhang2018retgk} and transportation distances \cite{solomon2016continuous}. Our code is available at {\small \url{https://github.com/IBM/RandomGraphEmbeddings}}. 

We highlight the main contributions as follows: 
\begin{itemize}
    \item We propose a class of p.d. global alignment graph kernels based on their global properties derived from geometric node embeddings and the corresponding node transportation.
    \item We present \emph{Random Graph Embeddings (RGE)}, a by-product of the RF approximation, which yields an expressive graph embedding. Based on this graph embedding, we significantly reduce computational complexity at least \emph{from quadratic to (quasi-)linear} in both the number and the size of the graphs.
    \item We theoretically show the uniform convergence of RGE. We prove that given $\Omega(1/\epsilon^2)$ random graphs, the inner product of RGE can uniformly approximates the corresponding exact graph kernel within $\epsilon-$precision, with high probability.
    %We theoretically show that uniform convergence within $\epsilon$ precision of an exact graph kernel is guaranteed by RGE using $\Omega(1/\epsilon^2)$ random graphs.
    %we provide an analysis showing that RGE guarantees uniform convergence to within $\epsilon$ precision of the exact graph kernel with a number $R=\Omega(1/\epsilon^2)$ of random graphs. 
    \item Our experimental results on nine benchmark datasets demonstrate that RGE outperforms or matches twelve state-of-the-art graph classification algorithms including graph kernels and deep graph neural networks. In addition, we numerically show that RGE can achieve (quasi-)linear scalability with respect to both the number and the size of graphs. 
    %Our experiments on nine graph kernel and social network benchmark datasets demonstrate that RGE matches or outperforms twelve state-of-the-art graph kernel and deep graph neural networks. In addition, RGE has shown to achieve (quasi-)linear scalability when increasing the number of the graphs and graph size. 
\end{itemize}

%%%%%%%%%%%%%%%%%%%%%%%%%%%%%%%%%%%%%%%%%%%%%%%%%%%%%%%%%%%%%%%%%%%%%%%%%%%%%%%
\section{Related Work}
In this section, we first make a brief survey of the existing graph kernels and then detail the difference between conventional random features method for vector inputs \cite{rahimi2008random} and our random features method for structured inputs. 

\vspace{-2mm}
\subsection{Graph Kernels}
Generally speaking, we can categorize the existing graph kernels into two groups: kernels based on local sub-structures, and kernels based on global properties. 

The first group of graph kernels compare sub-structures of graphs, following a general kernel-learning framework, i.e., R-convolution for discrete objects \cite{haussler1999convolution}. The major difference among these graph kernels is rooted in how they define and explore sub-structures to define a graph kernel, including 
random walks \cite{gartner2003graph}, 
shortest paths \cite{borgwardt2005shortest}, 
cycles \cite{horvath2004cyclic}, 
subtree patterns \cite{shervashidze2009fast}, 
and graphlets \cite{shervashidze2009efficient}. 
 A thread of research attempts to utilize node label information using the Weisfeiler-Leman (WL) test of isomorphism \cite{shervashidze2011weisfeiler} and takes structural similarity between sub-structures into account \cite{yanardag2015deep,yanardag2015structural} to further improve the performance of kernels. 

Recently, a new class of graph kernels, which focus on the use of geometric node embeddings of graph to capture global properties, are proposed. These kernels have achieved state-of-the-art performance in the graph classification task \cite{johansson2014global,johansson2015learning,nikolentzos2017matching}.  
The first global kernel was based on the Lov{\'a}sz number \cite{lovasz1979shannon} and its associated orthonormal representation \cite{johansson2014global}. However, these kernels can only be applied on unlabelled graphs. Later approaches directly learn graph embeddings by using landmarks \cite{johansson2015learning} or compute a similarity matrix \cite{nikolentzos2017matching} by exploiting different matching schemes between geometric embeddings of nodes of a pair of graphs.
Unfortunately, the resulting kernel matrix does not yield a \emph{valid p.d.} kernel and thus delivers a serious blow to hopes of using kernel support machine. Two recent graph kernels, the multiscale laplacian kernel \cite{kondor2016multiscale} and optimal assignment kernel \cite{kriege2016valid} were developed to overcome these limitations by building a p.d. kernel between node distributions or histogram intersection.

However, most of existing kernels only focus on learning kernel matrix for graphs instead of graph-level representation, which can only be used for graph classification rather than other graph related tasks (e.g., graph matching). More importantly, how to align the nodes in two graphs plays a central role in learning a similarity score. In this paper, we rely on an optimal transportation distance (e.g., Earch Mover's Distance) to learn the alignment between corresponding nodes that have similar structural roles in graphs, and directly generate a graph-level representation (embedding) for each graph instead of explicitly computing a kernel matrix.

\subsection{Random Features for Kernel Machines}
\label{subsec:random features}
Over the last decade, the most popular approaches for scaling up kernel method is arguably random features approximation and its fruitful variants \cite{rahimi2008random,sinha2016learning,bach2017equivalence,wu2018scalable}. Given a predefined kernel function, the inner product of RF directly approximates the exact kernel via sampling from a distribution, which leads to a fast linear method for computing kernel based on the learned low-dimensional feature representation. However, these RF approximation methods can only be applied to the shift-invariant kernels (e.g., the Gaussian or Laplacian kernels) with vector-form input data. Since a graph is a complex object, the developed graph kernels are neither shift-invariant kernels nor with vector-form inputs. Due to these challenges, to the best of our knowledge, there are no existing studies on how to develop the RF approximation for graph kernels. 

A recent work, called D2KE (distances to kernels and embeddings)~\cite{wu2018d2ke}, proposes the general methodology of the derivation of a positive-definite kernel through a RF map from any given distance function, which enjoys better theoretical guarantees than other distance-based methods. In \cite{wu2018random}, D2KE was extended to design a specialized time-series embedding and showed the strong empirical performance for time-series classification and clustering. We believe there is no work on applying D2KE to the graph kernel domain 
\footnote{Upon acceptance of this paper, a parallel work \cite{al2019ddgk} also adopted D2KE to develop an unsupervised neural network model for learning graph-level embedding.}. 
Our work is the first one to build effective and scalable global graph kernels using Random Features. 
%Conceptually, just like most of graph kernels can be viewed as instances to R-convolution \cite{haussler1999convolution}, our graph kernel can also be treated as an instance to D2KE \cite{wu2018d2ke}.

%%%%%%%%%%%%%%%%%%%%%%%%%%%%%%%%%%%%%%%%%%%%%%%%%%%%%%%%%%%%%%%%%%%%%%%%%%%%%%%
\section{Geometric Embeddings of Graphs and Earth Mover's Distance}
In this section, we will introduce two important building blocks of our method, the geometric node embeddings that are used to represent a graph as a bag-of-vectors, and the well-known transportation distance EMD. 

% \textbf{Geometric Embeddings of Graphs.} 
\subsection{Geometric Embeddings of Graphs}
The following notation will be used throughout the paper. Let a graph consisting of $n$ nodes, $m$ edges, and $l$ discrete node labels be represented as a triplet $G = (V, E, \ell)$, where $V = \{ v_i \}_{i=1}^n$ is the set of vertices, $E \subseteq (V \times V)$ is the set of undirected edges, and $\ell: V \rightarrow \Sigma$ is a function that assigns the label information to nodes from an alphabet set $\Sigma$. In this paper, we will consider both unlabeled graphs and graphs with discrete node labels.
%For simplicity, we assume that each graph has $n$ nodes, $m$ edges, and $l$ discrete node labels.
Let $\GG$ be a set of $N$ graphs where $\GG = \{G_i\}_{i=1}^N$ and let $\Y$ be a set of graph labels 
\footnote{Note that there are two types of labels involved in our paper, i.e., the node labels  and the graph labels. The node labels characterize the property of nodes. The graph labels are the classes that graph belongs to.}
corresponding to each graph in $\GG$ where $\Y = \{Y_i\}_{i=1}^N$. 
Let the geometric embeddings of a graph $G$ be a set of vectors $U = \{ \bu_i \}_{i=1}^n \in \R^{n \times d}$ for all nodes, where the vector $\bu_i$ in $U$ is the representation of the node $v_i$, and $d$ is the size of latent node embedding space. 

Typically, with different underlying learning tasks, a graph $G$ can be characterized by different forms of matrices. Without loss of generality, we use the normalized Laplacian matrix $L = D^{-1/2}(D-A)D^{-1/2} = I - D^{-1/2}AD^{-1/2}$, where $A$ is the adjacency matrix with $A_{ij} = 1$ if $(v_i, v_j) \in E$ and $A_{ij} = 0$ otherwise, and $D$ is the degree matrix. We then compute the $d$ smallest eigenvectors of $L$ to obtain $U $ as its geometric embeddings through the partial eigendecomposition of $L = U \Lambda U^T$. Then each node $v_i$ will be assigned an embedding vector $\bu_i \in \R^d$ where $\bu_i$ is the i-th row of the absolute $U$. Note that since the signs of the eigenvectors are arbitrary, we use the absolute values. Let $\bu_{ij}$  be the $j$th item of the vector $\bu_i$, then it satisfies $\vert \bu_{ij}\vert\leq 1$. Therefore, the node embedding vectors can be viewed as points in a d-dimensional unit hypercube. This fact plays an important role in our following sampling strategy. 

Note that although the standard dense eigensolvers require at least cubic time complexity in the number of graph nodes, with a state-of-the-art iterative eigensolver \cite{stathopoulos2010primme,wu2017primme_svds}, we can efficiently solve eigendecomposition with complexity that is linear in the number of graph edges.
%$O(dmz)$ in the number of graph edges, where $z$ is the number, typically quite small, of iterations.
It is also worth noting that the resulting geometric nodes embeddings well capture global properties of the graph since the eigenvectors associated with low eigenvalues of $L$ encode the information about the overall structure of $G$ based on the spectral graph theory \cite{von2007tutorial}.
%Interestingly, the resulting embeddings of nodes capture the global properties of a graph since the global connectivity of a graph can be measured by the spectrum of the adjacency matrix associated to eigenvector centrality \cite{jackson2010social}. 
%we then represent each graph as bag-of-vectors using a set of geometric node embeddings of each graph. However, since there is canonical ordering for the nodes of a graph, it is important to find an optimal matching between two sets of node embeddings when comparing two graphs.

In the traditional model of Natural Language Processing, a bag-of-words had been the most common way to represent a document. With modern deep learning approaches, each element such as a word in the document or a character in the string is embedded into a low-dimensional vector and is fed within a bag-of-vectors into recurrent neural networks that perform document and string classification. Similarly, we also represent each graph as bag-of-vectors using a set of geometric node embeddings. However, although there is canonical ordering for the nodes of a graph, it is not reliable in most case. Therefore, it is important to find an optimal matching between two sets of node embeddings when comparing two graphs.

% \textbf{Node Transportation via Earth Mover's Distance.}
\subsection{Node Transportation via EMD}
Now we assume that a graph $G$ is represented by the bag-of-vectors $\{\bu_1, \bu_2, \ldots, \bu_n\}$. To use the bag-of-words model, we also need to compute weights associated with each node vector. To be precise, if node $v_i$ has $c_i$ outgoing edges, we use $\bt_i = (c_i / \sum_{j=1}^n c_j) \in \R$ as a normalized bag-of-words (nBOW) weight for each node. Our goal is to measure the similarity between a pair of graphs $(G_i, G_j)$ using a proper distance measure. Instead of treating it as an assignment problem solved by maximum weight matching as in \cite{johansson2015learning}, we cast the task as a well-known transportation problem \cite{hitchcock1941distribution}, which can be addressed by using the Earth Mover's Distance \cite{rubner2000earth}. 

Using EMD, one can easily measure the dissimilarity between a pair of graphs $(G_x, G_y)$ through node transportation, which essentially takes into account alignments between nodes.
Let $n = max(n_x,n_y)$ denote the maximum number of nodes in a pair of graphs $(G_x, G_y)$. Since $\bt^{(G_x)}$ is the nBOW weight vector for the graph $G_x$, it is easy to obtain that $(\bt^{(G_x)})^T \1 = 1$. Similarly, we have $(\bt^{(G_y)})^T \1 = 1$. Then the EMD is defined as
\begin{equation}\label{EMD}
\begin{aligned}
& \EMD(G_x,G_y) := \min_{\T \in \R_+^{n_x \times n_y}} \langle \D, \T \rangle, \\
& \text{subject to}: \T\1=\bt^{(G_x)},\;\;\T^T\1=\bt^{(G_y)}.
\end{aligned}
\end{equation}
% \begin{equation}\label{EMD}
% \EMD(G_x,G_y) := \min_{\T \in \R_+^{n_x \times n_y}} \langle \D, \T \rangle, 
% s.t.,    \T\1=\bt^{(G_x)},\;\;\T^T\1=\bt^{(G_y)}.
% \end{equation}
where $\T$ is the transportation flow matrix with $\T_{ij}$ denoting how much of node $v_i$ in $G_x$ travels to node $v_j$ in $G_y$, and $\D$ is the transportation cost matrix where each item $\D_{ij}=d(\bu_{i},\bu_{j})$ denotes the distance between two nodes measured in their embedding space. Typically, the Euclidean distance $d(\bu_{i},\bu_{j})=\|\bu_{i}-\bu_{j}\|_2$ is adopted. 
We note that with the distance $d(\bu_{i},\bu_{j})$ is a \emph{metric} in the embedding space, the EMD \eqref{EMD} also define a \emph{metric} between two graphs \cite{rubner2000earth}. An attractive attribute of the EMD is that it provides an accurate measurement of the distance between graphs with different nodes that are contextually similar but in different positions in the graph. The EMD distance has been observed to perform well on text categorization \cite{kusner2015word} and graph classification \cite{nikolentzos2017matching}. A straightforward way that defines a kernel matrix based on EMD that measures the similarity between graphs has been shown in \cite{nikolentzos2017matching} as follows:
\begin{equation} \label{eq:sim_EMD}
K = - \frac{1}{2} J D J
% f(G_x,G_y)=\exp(-\gamma \text{EMD}(G_x,G_y)).
\end{equation}
where $J$ is the centering matrix $J = I - \frac{1}{N}\1 \1^T$ and $D$ is the EMD distance matrix from all the pairs of graphs. 
However, there are three problems. The first one is that the Kernel matrix in \eqref{eq:sim_EMD} is not necessarily positive-definite. The second problem is that the EMD is expensive to compute, since its time complexity is $O(n^3\mathrm{log}(n))$. In addition, computing the EMD for each pair of graphs requires the quadratic time complexity $O(N^2)$ in the number of graphs, which is highly undesirable for large-scale graph data. In this paper, we propose a scalable global alignment graph kernel using the random features to simultaneously address all these issues.
%When the total weights of two graphs are equal and $d(\bu_{i},\bu_{j})$ is a \emph{metric} in the Euclidean distance, the EMD \eqref{EMD} qualifies as a metric, meaning that it satisfies the \emph{triangle inequality} \cite{rubner2000earth}. An attractive attribute of the EMD is that it is supports accurate measurement of the distance between graphs with different nodes that are contextually similar but in different positions. The EMD distance has been observed to perform well on text categorization \cite{kusner2015word} and graph classification \cite{nikolentzos2017matching}. However, the EMD is expensive to compute; its computational complexity is $O(n^3 \log(n))$, and for large graphs, $n$ is large. More importantly, building a kernel matrix using EMD does not lead to a positive p.d. kernel, which impairs its use and performance.   

% \begin{wrapfigure}{r}{0.5\textwidth}
% \includegraphics[width=0.48\textwidth]{Graphs/RGE_EMD_Demo/RGE_EMD.pdf}
% \caption{An illustration of how the EMD is used to measure the distance between a random graph and a raw graph. Each small random graph implicitly partitions the larger raw graph through node alignments in a low dimensional node embedding space.} 
% \label{fig:rge_emd_demo}
% \end{wrapfigure} 

%%%%%%%%%%%%%%%%%%%%%%%%%%%%%%%%%%%%%%%%%%%%%%%%%%%%%%%%%%%%%%%%%%%%%%%%%%%%%%%
% \section{Earth Mover's Distance Based Global Graph Kernel}
\section{Scalable Global Alignment Graph Kernel Using Random Features}
%In this section, we extend the generic framework in \cite{wu2018d2ke}
In this section, we first show how to construct a class of the p.d. global alignment graph kernels from an optimal transportation distance (e.g., EMD) and then present a simple yet scalable way to compute expressive graph embeddings through the RF approximation. We also show that the inner product of the resulting graph embeddings uniformly converge to the exact kernel. 

% \begin{figure}[tb]
% \centering
% \includegraphics[scale = 0.38]{Graphs/RGE_EMD_Demo/RGE_EMD.pdf}
% \vspace{-2mm}
% \caption{An illustration of how the EMD is used to measure the distance between a random graph and a raw graph. Each small random graph implicitly partitions the larger raw graph through node alignments in a low dimensional node embedding space.} 
% \vspace{-1mm}
% \label{fig:rge_emd_demo}
% \end{figure}

\subsection{Global Alignment Graph Kernel Using EMD and RF}
The core task is to build a \emph{positive-definite} graph kernel that can make full use of both computed geometric node embeddings for graphs and a distance measure considering the alignment of the node embeddings. We here define our global graph kernel as follows:
\begin{equation}\label{eq:GGK_EMD}
\begin{aligned}
&k(G_x,G_y):=  \int p(G_\omega) \phi_{G_\omega}(G_x)\phi_{G_\omega}(G_y) dG_\omega, \\
& \text{where} \;\; \phi_{G_\omega}(G_x):=\exp(-\gamma \EMD(G_x,G_\omega)).
\end{aligned}
\end{equation}
% \begin{equation}\label{eq:GGK_EMD}
% k(G_x,G_y):=  \int p(G_\omega) \phi_{G_\omega}(G_x)\phi_{G_\omega}(G_y) dG_\omega, 
% \text{where} \;\; \phi_{G_\omega}(G_x):=\exp(-\gamma \EMD(G_x,G_\omega)).
% \end{equation}
Here $G_\omega$ is a random graph consisting of $D$ random nodes with their associated node embeddings $W = \{ \bw_i \}_{i=1}^D$,  where each random node embedding $\bw_i$ is sampled from a d-dimensional vector space $\V \in \R^{d}$. Thus, $p(G_\omega)$ is a distribution over the space of all random graphs of variable sizes $\Omega:=\bigcup_{D=1}^{D_{max}}\V^{D}$. Then we can derive an infinite-dimensional feature map $\phi_{G_\omega}(G_x)$ from the EMD between $G_x$ and all possible random graphs $G_\omega \in \Omega$. One explanation of how our proposed kernel works is that a small random graph  can implicitly partition a larger raw graph through node transportation (or node alignments) in the corresponding node embedding space using EMD, as illustrated in Fig. \ref{fig:rge_emd_demo}. 

\begin{figure}[tb]
\centering
\includegraphics[scale = 0.38]{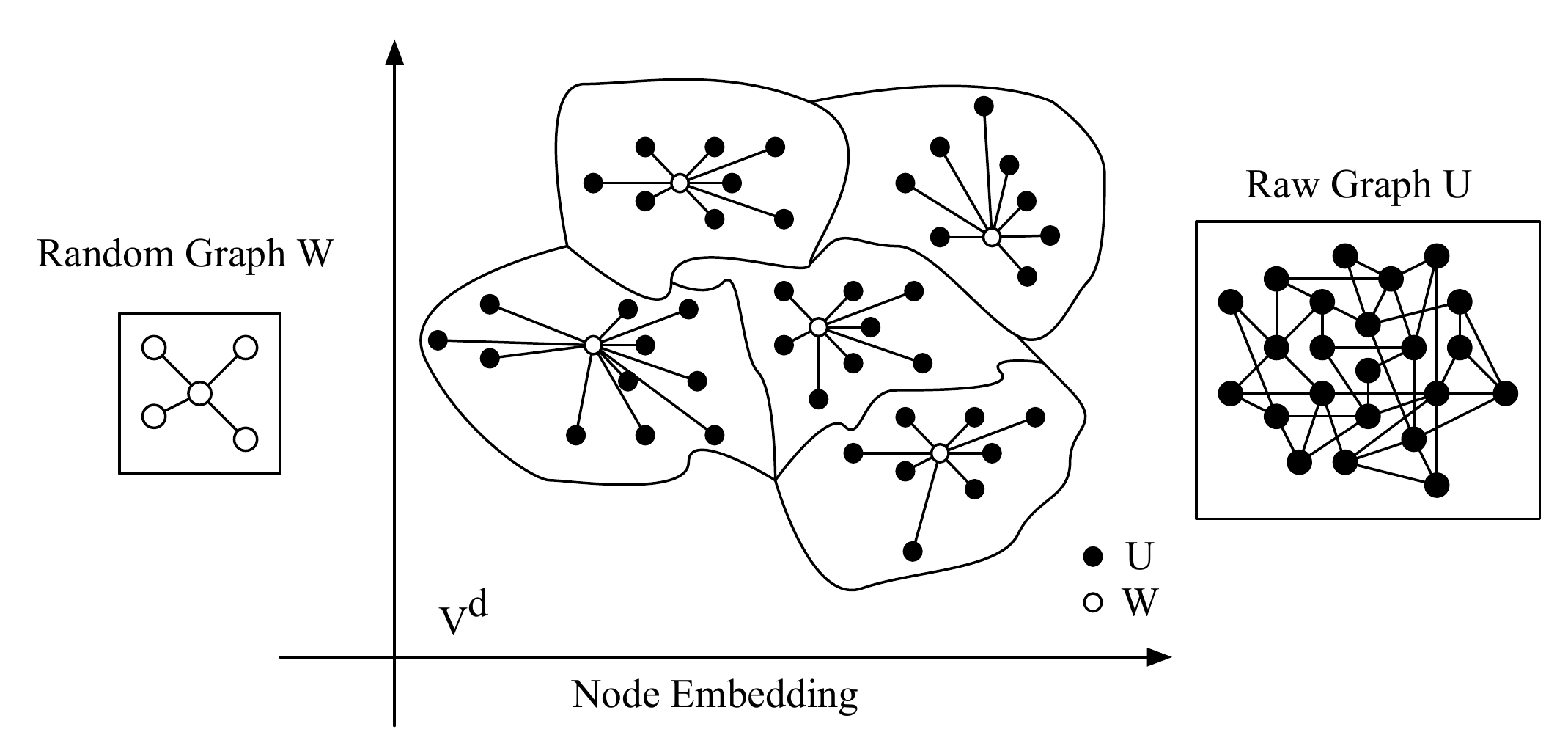}
\vspace{-2mm}
\caption{An illustration of how the EMD is used to measure the distance between a random graph and a raw graph. Each small random graph implicitly partitions the larger raw graph through node alignments in a low dimensional node embedding space.} 
\vspace{-4mm}
\label{fig:rge_emd_demo}
\end{figure}

A more formal and revealing way to interpret our kernel defined in \eqref{eq:GGK_EMD} is to express it as
\begin{equation}\label{eq:GGK_EMD2}
\begin{aligned}
& k(G_x,G_y) :=  \\
& \exp\left( -\gamma\softmin_{p(G_\omega)}\{ \EMD(G_x,G_\omega)+\EMD(G_\omega,G_y) \} \right)
\end{aligned}
\end{equation}
% \begin{equation}\label{eq:GGK_EMD2}
% k(G_x,G_y) := 
% \exp\left( -\gamma\softmin_{p(G_\omega)}\{ \EMD(G_x,G_\omega)+\EMD(G_\omega,G_y) \} \right)
% \end{equation}
where,
\begin{equation}\label{softmin}
 \softmin_{p(G_\omega)}\{f(G_\omega)\}:= -\frac{1}{\gamma}\log \int p(G_\omega) e^{-\gamma f(G_\omega)} dG_\omega
\end{equation} 
can be treated as the soft minimum function defined by two parameters $p(G_\omega)$ and $\gamma$. Since the usual soft minimum is defined as $\softmin_i f_i := -\softmax_i(-f_i) = -\log \sum_{i} e^{-f_i}$, then Equation \eqref{softmin} can be regarded as its smoothed version, which uses parameter $\gamma$ to control the degree of smoothness and is reweighted by a probability density $p(G_\omega)$. Interestingly, the value of \eqref{softmin} is mostly determined by the minimum of $f(G_\omega)$, when $f(G_\omega)$ is Lipschitz-continuous and $\gamma$ is large. Since EMD is a metric as discussed above, we have
\begin{equation*}\label{eq:GGK_EMD tri_ineq}
 \EMD(G_x,G_y) \leq  
 \min_{G_\omega\in \Omega} \left(\EMD(G_x,G_\omega)+\EMD(G_\omega,G_y)\right)
\end{equation*}
by the triangle inequality. The equality holds if the maximum size of the random graph, $D_{\max}$, is equal or greater than the original graph size $n$. Therefore, the kernel value in \eqref{eq:GGK_EMD2} serves as a good approximation to the EMD between any pair of graphs $G_x$ and $G_y$. By the kernel definition, it must be \emph{positive-definite}.

\subsection{Random Graph Embedding: Random Features of Global Alignment Graph Kernel}
\label{sec:random graph embeddings}
In this section, we will introduce how to efficiently compute the proposed global alignment graph kernels and derive the random graph embedding that can be used for representing graph-level embedding from the geometric node embeddings. 
\subsubsection{Efficient Computation of RGE} \label{subsec:efficient computation of RGE}
\hfill\\
% \textbf{Efficient Computation of RGE.} 
Exact computation of the proposed kernel in \eqref{eq:GGK_EMD} is often infeasible, as it does not admit a simple analytic solution. A natural way to compute such kernel is to resort to a kernel approximation that is easy to compute while uniformly converges to the exact kernel. As one of the most effective kernel approximation techniques, random features method has been demonstrated great successes in approximating Gaussian Kernel \cite{rahimi2008random,le2013fastfood} and Laplacian Kernel \cite{wu2016revisiting} in various applications. However, as we discussed before in Sec. \ref{subsec:random features}, conventional RF methods cannot be directly applied to our graph kernels since they are not shift-invariant and cannot deal with the inputs that are not vector-form. Moreover, for traditional RF methods, we have to know the kernel function prior before hand, which is also not available in our case. 
However, fortunately, since we can define our kernel in terms of a randomized feature map, it naturally yields the following random approximation that does not require aforementioned assumptions,
% \begin{equation}\label{RGE}
% \begin{aligned}
% k(G_x,G_y) & \approx \langle Z(G_x), Z(G_y) \rangle \\
% & = \frac{1}{R}\sum_{i=1}^R \phi_{{G_\omega}_i}(G_x)\phi_{{G_\omega}_i}(G_y)
% \end{aligned}
% \end{equation}
% \begin{equation}\label{RGE}
% k(G_x,G_y) \approx \langle Z(G_x), Z(G_y) \rangle 
% = \frac{1}{R}\sum_{i=1}^R \phi_{{G_\omega}_i}(G_x)\phi_{{G_\omega}_i}(G_y)
% \end{equation}
% where $Z(G_x):=(\frac{1}{\sqrt{R}}\phi_{{G_\omega}_i}(G_x))_{i=1}^R$ yields a vector-form graph-level representation of graph $G_x$, and $\{{G_\omega}_i\}_{i=1}^R$ are drawn from $p(G_\omega)$ and can be easily seen as i.i.d. random graphs. We call this random approximation \emph{Random Graph Embedding (RGE)}, a generalized concept of "Random Features" for our graph inputs. We will also show that this random approximation RGE admits an uniform convergence to the original kernel \eqref{eq:GGK_EMD} over all pairs of graphs $(G_x,G_y)$. 
\begin{equation}\label{RGE}
\begin{aligned}
\tilde{k}(G_x,G_y)&=\langle Z(G_x), Z(G_y) \rangle 
= \frac{1}{R}\sum_{i=1}^R \phi_{{G_\omega}_i}(G_x)\phi_{{G_\omega}_i}(G_y)\\
&\rightarrow k(G_x,G_y),\ \mathrm{as}\ R\rightarrow\infty.
\end{aligned}
\end{equation}
where $Z(G_x)$ is a $R$-dimensional vector with the $i-$th term $Z(G_x)_i=\frac{1}{\sqrt{R}}\phi_{{G_\omega}_i}(G_x)$, and $\{{G_\omega}_i\}_{i=1}^R$ are i.i.d. samples drawn from $p(G_\omega)$. Note that the vector $Z(G_x)$ just can be considered as the representation (embedding) of graph $G_x$. We call this random approximation "\emph{random graph embedding (RGE)}", a generalized concept of "random features" for our graph inputs. We will also show that this random approximation RGE admits the uniform convergence to the original kernel \eqref{eq:GGK_EMD} over all pairs of graphs $(G_x,G_y)$.

Algorithm \ref{alg:EMD_features} summarizes the procedure to generate feature vectors for data graphs. There are several comments to make here. First of all, the distribution $p(G_\omega)$ is the key to generating high-quality node embeddings for random graphs.
%, capable of capturing the global structure information of nodes in raw graphs in the corresponding node embedding space. 
We propose two different ways to generate random graphs, which we will illustrate in detail later. Second, the size $D$ of the random graphs is typically quite small. An intuitive explanation why a small random graph captures important global information of raw graphs has been discussed in the previous section. However, since there is no prior information to determine how many random nodes is needed to segment the data graph for learning discriminatory features, we sample the size of the random graphs from a uniform distribution $[1, D_{max}]$ to obtain an unbiased estimate of $D$. Finally, both node embedding and distance measures can be further improved by exploiting the latest advancements in these techniques.

\begin{algorithm}[htb]
\caption{Random Graph Embedding}
\label{alg:EMD_features}
\begin{algorithmic}[1]
    \STATEx {\bf Input:} Data graphs $\{G_i\}_{i=1}^N$, node embedding size $d$, maximum size of random graphs $D_{max}$, graph embedding size $R$.
    \STATEx {\bf Output:} Feature matrix $Z_{N \times R}$ for data graphs
    \STATE Compute nBOW weights vectors $\{\bt^{(G_i)}\}_{i=1}^N$ of the normalized Laplacian $L$ of all graphs
    \STATE Obtain node embedding vectors $\{\bu_i\}_{i=1}^n$ by computing $d$ smallest eigenvectors of $L$
    \FOR {$j = 1, \ldots, R$}
        \STATE Draw $D_j$ uniformly from $[1, D_{max}]$. 
        \STATE Generate a random graph ${G_\omega}_j$ with $D_j$ number of  nodes embeddings $W$ from Algorithm \ref{alg:RG_gen}. 
        \STATE Compute a feature vector $Z_j = \phi_{{G_\omega}_j}(\{G_i\}_{i=1}^N))$ using EMD or other optimal transportation distance in Equation \eqref{eq:GGK_EMD}.
    \ENDFOR
    \STATE Return feature matrix $Z(\{G_i\}_{i=1}^N) = \frac{1}{\sqrt{R}} \{Z_i\}_{i=1}^R$
\end{algorithmic}
\end{algorithm}
\vspace{-0mm}

By efficiently approximating the proposed global alignment graph kernel using RGE, we obtain the benefits of both improved accuracy and reduced computational complexity. Recall that the computation of EMD has time complexity $O(n^3log(n))$ and thus the existing graph kernels require at least $O(N^2 n^3log(n))$ computational complexity and $O(N^2)$ memory consumption, where $N$ and $n$ are the number of graphs and the average size of graphs, respectively. Because of the small size of random graphs, the computation of EMD in our RGE approximation only requires $O(D^2nlog(n))$ \cite{bourgeois1971extension}. It means that our RGE approximation only requires computation with the quasi-linear complexity $O(nlog(n))$ if we treat $D$ as a constant (or a small number). Note that with a state-of-the-art eigensolver \cite{stathopoulos2010primme,wu2017primme_svds}, we can effectively compute the $d$ largest eigenvectors with linear complexity $O(dmz)$, where $m$ is the number of graph edges and $z$ is the number, typically quite small, of iterations of iterative eigensolver. Therefore, the total computational complexity and memory consumption of RGE are $O(NRnlog(n) + dmz)$ and $O(NR)$ respectively. Compared to other graph kernels, our method reduces computational complexity from quadratic to linear in terms of the number of graphs, and from (quasi-)cubic to (quasi-)linear in terms of the graph size. We will empirically assess the computational runtime in the subsequent experimental section. 

\subsubsection{Data-independent and Data-dependent Distributions}
\hfill\\
% \textbf{Data-independent and Data-dependent Distributions.} 
Algorithm \ref{alg:RG_gen} details the two sampling strategies (data-independent and data-dependent distributions) for generating a set of node embeddings of a random graph.
% We explore two different means to generate high-quality random graphs. 
The first scheme is to produce random graphs from a data-independent distribution. Traditionally, conventional RF approximation has to obtain random samples from a distribution corresponding to the user predefined kernel (e.g., Gaussian or Laplacian kernels). However, since we reverse the order by firstly defining the distribution and then defining a kernel similar to \cite{wu2018d2ke}, we are free to select any distribution that can capture the characteristics of the graph data well. Given that all node embeddings are distributed in a $d$-dimensional unit hypercube space, we first compute the largest and smallest elements in all node embeddings and then use a uniform distribution in the range of these two values to generate a set of $d$-dimensional vectors for random node embeddings in a random graph. Since node embeddings are roughly dispersed uniformly in the $d$-dimensional unit hypercube space, we found this scheme works well in most of cases. Like the traditional RF, this sampling scheme is data-independent. So we call it RGE(RF). 

Another scheme is conceptually similar to recently proposed work on deriving data-dependent traditional random features \cite{ionescu2017large} for vector-inputs, which have been shown to have a lower generalization error than data-independent random features \cite{rudi2017generalization}. However, unlike these conventional RF methods \cite{ionescu2017large,rudi2017generalization} and the conventional landmarks method that selects a representative set of whole graphs \cite{johansson2015learning}, we propose a new way to sample parts of graphs (only from training data) as random graphs, which we refer to as the \emph{Anchor Sub-Graphs (ASG)} scheme RGE(ASG). There are several potential advantages compared to lankmarks and RF methods. First of all, ASG opens the door to defining an indefinite feature space since there are conceptually unlimited numbers of sub-graphs, compared to the limited size (up to the number of graphs) of landmarks. Second, ASG produces a random graph by permuting graph nodes of the original graph and by resembling randomly their corresponding node embeddings in the node embedding space, which may help to identify more hidden global structural information instead of only considering the raw graph topology. Thanks to EMD, hidden global structure can be captured well through node alignments. Finally, unlike RGE(RF), the ASG scheme allows exploiting node-label information in raw graphs since this information is also accessible through the sampled nodes in sub-graphs.

Incorporating the node label information into RGE(ASG) is fairly straightforward; it is desirable to assign nodes with same labels a smaller distance than these with different labels. Therefore, we can simply set the distance $d(\bu_i,\bu_j) = \textit{max}(\|\bu_{i}-\bu_{j}\|_2, \sqrt{d})$ if nodes $v_i$ and $v_j$ have different node labels since \emph{$\sqrt{d}$} is the largest distance in a $d$-dimensional unit hypercube space. 
%Algorithm \ref{alg:RG_gen}  detail the two sampling strategies (RF and ASG) for generating a set of node embeddings of a random graph. 

\begin{algorithm}[tb]
\caption{Random Graph Generation}
\label{alg:RG_gen}
\begin{algorithmic}[1]
    \STATEx {\bf Input:} Node embeddings $U = \{\bu_i\}_{i=1}^n$, node embedding size $d$, size of random graph $D_j$.
    \STATEx {\bf Output:} Random node embeddings $W = \{\bw_i\}_{i=1}^{D_j}$
    \IF{Choose RGE(RF)}
        \STATE Compute maximum value $u_{max}$ and minimum value $u_{min}$ in $U$. 
        \STATE Generate a number $D_j$ of random node embedding vectors $\{\bw_i\}_{i=1}^{D_j}$ in a random graph drawn from $ \left( u_{min}+(u_{max}-u_{min})\times rand(d,D_j) \right)$.
    \ELSIF{Choose RGE(ASG)}
        \STATE Uniformly draw graph index $k = rand(1, N)$ and select the $k$-th raw graph 
        \STATE Uniformly draw a number $D_j$ of node indices $\{n_1, n_2, \ldots, n_{D_j}\}$ in the $k$-th raw graph 
        \STATE Generate a number $D_j$ of random node embedding vectors $\{\bw_i\}_{i=1}^{D_j} = \{ \bu_{n_1}, \bu_{n_2}, \ldots, \bu_{n_{D_j}} \} $ as well as its associated node labels for a random graph
    \ENDIF
    \STATE Return nodes embeddings $W = \{\bw_i\}_{i=1}^{D_j}$ for a random graph
\end{algorithmic}
\vspace{-0mm}
\end{algorithm}

\subsection{Convergence of Random Graph Embedding}
In this section, we establish a bound on the number of random graphs required to guarantee an $\epsilon$ approximation between the exact kernel \eqref{eq:GGK_EMD} and its random feature approximation \eqref{RGE} denoted by $\tk(G_x,G_y)$. We first establish a covering number for the space $\X$ under the EMD metric.

\begin{lemma} \label{lemma:cover}
There is an $\epsilon$-covering $\E$ of $\X$ under the metric defined by EMD with Euclidean ground distance such that
$$
\forall G\in\X, \exists G_i\in \E,\;\EMD(G,G_i)\leq \epsilon.
$$
with $|\E| \leq (1+\frac{2}{\epsilon})^{Md}$, where $M$ is an upper bound on the number of nodes for any graph $G\in\X$.
\end{lemma}

\begin{proposition}\label{coverrelation}
Let $\Delta_R(G_x,G_y)=k(G_x,G_y)-\tk(G_x,G_y)$. We have that if $|\Delta_R(G_i,G_j)|\leq t$, $\forall G_i,G_j\in\E$, where $\E$ is an $\frac{t}{4\gamma}-$covering of $\mathcal{X}$, and $\gamma$ is the parameter of $\phi_{G_{\omega}}$, then $|\Delta_R(G_x,G_y)|\leq 2t$, $\forall G_x,G_y \in \mathcal{X}$.
\end{proposition}

% \begin{proof}
% Firstly, we find an $\epsilon$-covering $\E_v$ of size $(1+\frac{2}{\epsilon})^{d}$ for the node embedding space. Then define $\E$ as all the possible sets of $\bv\in\E_v$ of size no larger than $M$. Then we have $|\E|\leq (\frac{2}{\epsilon})^{dM}$, and for any graph $G=(\bv_j)_{j=1}^n \in\X$, we can find $G_i\in \E$ with also $n$ nodes $(\bu_j)_{j=1}^n$ such that $\|\bu_j-\bv_j\|\leq \epsilon$. Then by the definition of EMD \eqref{EMD}, a solution that assigns each node $\bv_j$ in $G$ to a node $\bu_j$ in $x_i$ would have overall cost less than $\epsilon$, and therefore, $\EMD(G,G_i)\leq \epsilon$.
% \end{proof}

Thus, given $\Omega(\frac{1}{\epsilon^2})$ random graphs, the inner product of RGE can uniformly approximates the corresponding exact graph kernel within $\epsilon-$precision, with high probability, as shown in the following Theorem. 

\begin{theorem}\label{thm:RF}
The uniform convergence rate is
\begin{equation*}\label{converge_result}
P\left\{ \sup_{G_x,G_y\in\X} |\Delta_R(G_x,G_y)| \leq \epsilon\right\} \geq 1-2(1+\frac{16\gamma}{\epsilon})^{2dM}\exp(-R\epsilon^2/8).
\end{equation*}
Therefore, to guarantee $|\Delta_R(G_x,G_y)|\leq \epsilon$ with probability at least $1-\delta$, it suffices to have
$$
R = \Omega\biggl(\frac{Md}{\epsilon^2}\log(1+\frac{16\gamma}{\epsilon})+\frac{1}{\epsilon^2}\big[\log(\frac{1}{\delta})+\text{const}\big] \biggr).
$$
\end{theorem}

% \begin{theorem}\label{thm:RF}
% Let $\Delta_R(G_x,G_y)=k(G_x,G_y)-\tk(G_x,G_y)$. We have uniform convergence
% \begin{equation*}\label{converge_result}
% P\left\{ \max_{G_x,G_y\in\X} |\Delta_R(G_x,G_y)| > 2t\right\} \leq 2\left(\frac{12\gamma}{t}\right)^{2Md}e^{-Rt^2/2}.
% \end{equation*}
% Therefore, to guarantee $|\Delta_R(G_x,G_y)|\leq \epsilon$ with probability at least $1-\delta$, it suffices to have
% $$
% R = \Omega\biggl(\frac{Md}{\epsilon^2}\log(\frac{\gamma}{\epsilon})+\frac{1}{\epsilon^2}\log(\frac{1}{\delta}) \biggr).
% $$
% \end{theorem}
\begin{proof}
Based on Proposition \ref{coverrelation}, we have 
\begin{equation}\label{them:eq1}
\begin{aligned}
&P\left\{\mathrm{sup}_{G_x, G_y\in\mathcal{X}}\vert\Delta_R(G_x,G_y)\vert\leq 2t\right\}\\
\geq &P\left\{\mathrm{sup}_{G_i, G_j\in\E}\vert\Delta_R(G_x,G_y)\vert\leq t\right\}.
\end{aligned}
\end{equation}
For any $G_i, G_j\in \E$, since $E[\Delta_R(G_i,G_j)]=0$ and $|\Delta_R(G_i,G_j)|\leq 1$, from the Hoeffding inequality, we have
\begin{equation}
P\left\{ |\Delta_R(G_i,G_j)|\geq t \right\} \leq 2 \exp(-Rt^2/2).
\end{equation}
Therefore,
\begin{equation}\label{them:eq2}
\begin{aligned}
&P\left\{\mathrm{sup}_{G_i, G_j\in\E}\vert\Delta_R(G_i,G_j)\vert\geq t\right\}\\
\leq&\sum_{G_i, G_j\in\E}P\{\vert\Delta_R(G_i,G_j)\vert\geq t\}\\
\leq &2\vert\E\vert^2\exp(-Rt^2/2)\leq2(1+\frac{8\gamma}{t})^{2dM}\exp(-Rt^2/2).
\end{aligned}
\end{equation}
Combining \eqref{them:eq1} and \eqref{them:eq2}, and setting $t=\frac{\epsilon}{2}$,  we obtain the desired result.
\end{proof}

The above theorem states that, to find an $\epsilon$ approximation to the exact kernel, it suffices to have number of random features 
%$R=\Omega(\frac{Md}{\epsilon^2})$. 
$R=\Omega(\frac{1}{\epsilon^2})$.
We refer interested readers to the details of the proof of Theorem \eqref{thm:RF} in Appendix \ref{appendix: proof of lemma and theory }.

%%%%%%%%%%%%%%%%%%%%%%%%%%%%%%%%%%%%%%%%%%%%%%%%%%%%%%%%%%%%%%%%%%%%%%%%%%%%%%%
\section{Experiments}
We performed experiments to demonstrate the effectiveness and efficiency of the proposed method, and compared against a total of twelve graph kernels and deep graph neural networks on nine benchmark datasets\footnote{http://members.cbio.mines-paristech.fr/~nshervashidze/code/} widely used for testing the performance of graph kernels. We implemented our method in Matlab and utilized the C-MEX function\footnote{http://ai.stanford.edu/$\sim$rubner/emd/default.htm} for the computationally expensive component of EMD. 
% To accelerate the computation, we use multithreading with total 12 threads in all experiments. 
All computations were carried out on a DELL system with Intel Xeon processors 272 at 2.93GHz for a total of 16 cores and 250 GB of memory, running the SUSE Linux operating system.

\textbf{Datasets.} 
We applied our method to widely-used graph classification benchmarks from multiple domains \cite{shervashidze2011weisfeiler,vishwanathan2010graph,yanardag2015deep}; MUTAG, PTC-MR, ENZYMES, PROTEINS, NCI1, and NCI109 are graphs derived from small molecules and macromolecules, and IMDB-B, IMDB-M, and COLLAB are derived from social networks. All datasets have binary labels except ENZYMES, IMDB-M, and COLLAB which have 6, 3 and 3 classes, respectively. All bioinformatics graph datasets have node labels while all other social network graphs have no node labels. Detailed descriptions of these 9 datasets, including statistical properties, are provided in the Appendix.
%We do not consider edge labels and we also leave comparisons to Graph Kernels using Weisfeler-Leman to future work. 
%Since WL test is a generic technique to be used for improving almost any stand-alone graph kernel, in this study, we focus on testing the true capability of each graph kernel and for this reason we leave the comparison with WL to future work. 
% 1) MUTAG: is a dataset of 188 mutagenic aromatic and heteroaromatic nitro compounds with 7 discrete node labels. Each chemical compound is labeled acooring to whether it has mutagenic effect on the gram-negative bacterium Salmonnella Typhimurium. 
% 2) PTC: is a dataset of 344 chemical compounds that reports the carcinogenicity for male and female arts and it has 19 node labels.
% 3) ENZYMES: is a dataset of 600 graphs reprensenting enzymes from the BRENDA enzyme database and it has 3 node labels. Each enzyme is a member of noe of the 6 Enzyme Commisson top level enzyme classes and the task is to correctly assign the enzymes to their classes. 
% 4) NCI1 and NCI109: are two balanced datasets (4100 and 4127 nodes respectively) of chemical compounds that are classified based on whether they are active or inactive against non-small cell lung cancer and ovarian cancer cell lines. The table \ref{tb:info of datasets} lists the statistical information about the datasets. 

\textbf{Baselines.}
Due to the large literature, we compare our method RGE against five representative global kernels related to our approach and three classical graph kernels, including
EMD-based Indefinite Kernel (EMD) \cite{nikolentzos2017matching},
Pyramid Match Kernel (PM) \cite{nikolentzos2017matching},
Lov{\'a}sz $\theta$ Kernel (Lo-$\theta$) \cite{johansson2014global},
Optimal Assignment Matching (OA-{$E_\lambda$}(A)) \cite{johansson2015learning},
Vertex Optimal Assignment Kernel (V-OA) \cite{kriege2016valid},
Random Walk Kernel (RW) \cite{gartner2003graph}, 
Graphlet Kernel (GL) \cite{shervashidze2009efficient}, and
Shortest Path Kernel (SP) \cite{borgwardt2005shortest}. Furthermore, we also compare RGE with several variants of Weisfeler-Leman Graph Kernel (WL-ST \cite{shervashidze2011weisfeiler}, WL-SP \cite{shervashidze2011weisfeiler}, and WL-OA-{$E_\lambda$}(A) \cite{johansson2015learning}). Finally, we compare RGE against four recently developed deep learning models with node labels, including 
Deep Graph Convolutional Neural Networks (DGCNN), \cite{wu2017dgcnn};
PATCHY-SAN (PSCN) \cite{niepert2016learning},
Diffusion CNN (DCNN) \cite{atwood2017sparse}, and
Deep Graphlet Kernel (DGK) \cite{yanardag2015deep}. The first three models are built on convolutional neural networks on graphs while the last one is based on Word2Vec model. 
Since WL test is a generic technique to utilize discrete node labels for improving many stand-alone graph kernels, in this study, we first focus on testing the capability of each graph kernel without node labels and then assess the performance of each graph kernel with plain node labels and with WL techniques
\footnote{Our approach to combine RGE(ASG) with WL techniques is to first use WL to generate new node labels and then apply RGE(ASG) with these node labels.}.

\begin{figure*}[!htbp]
\centering
    \begin{subfigure}[b]{0.245\textwidth}
      \includegraphics[width=\textwidth]{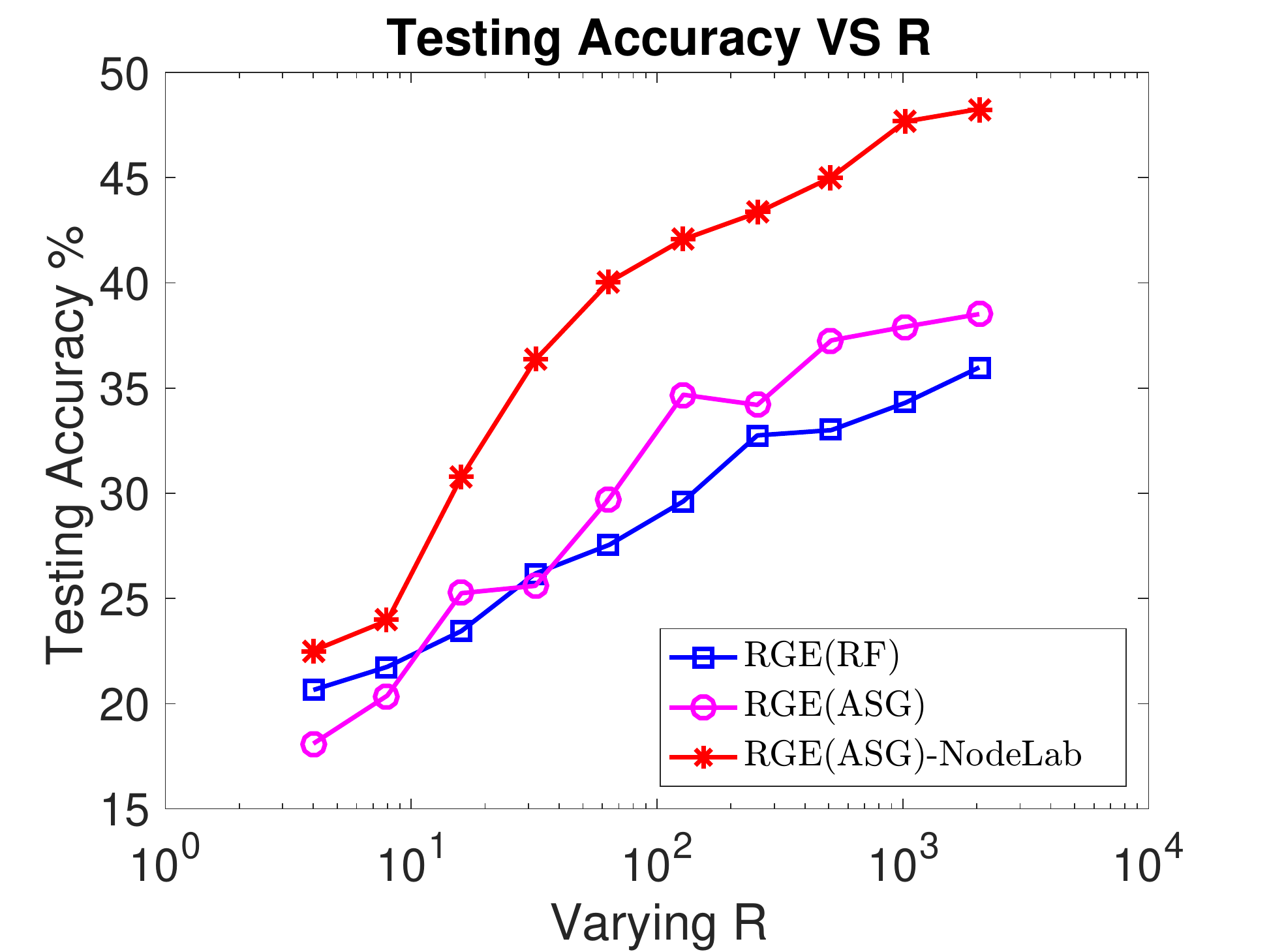}
      \caption{ENZYMES}
      \label{fig:exptsA_accu_varyingR_ENZYMES}
    \end{subfigure}
    \begin{subfigure}[b]{0.245\textwidth}
      \includegraphics[width=\textwidth]{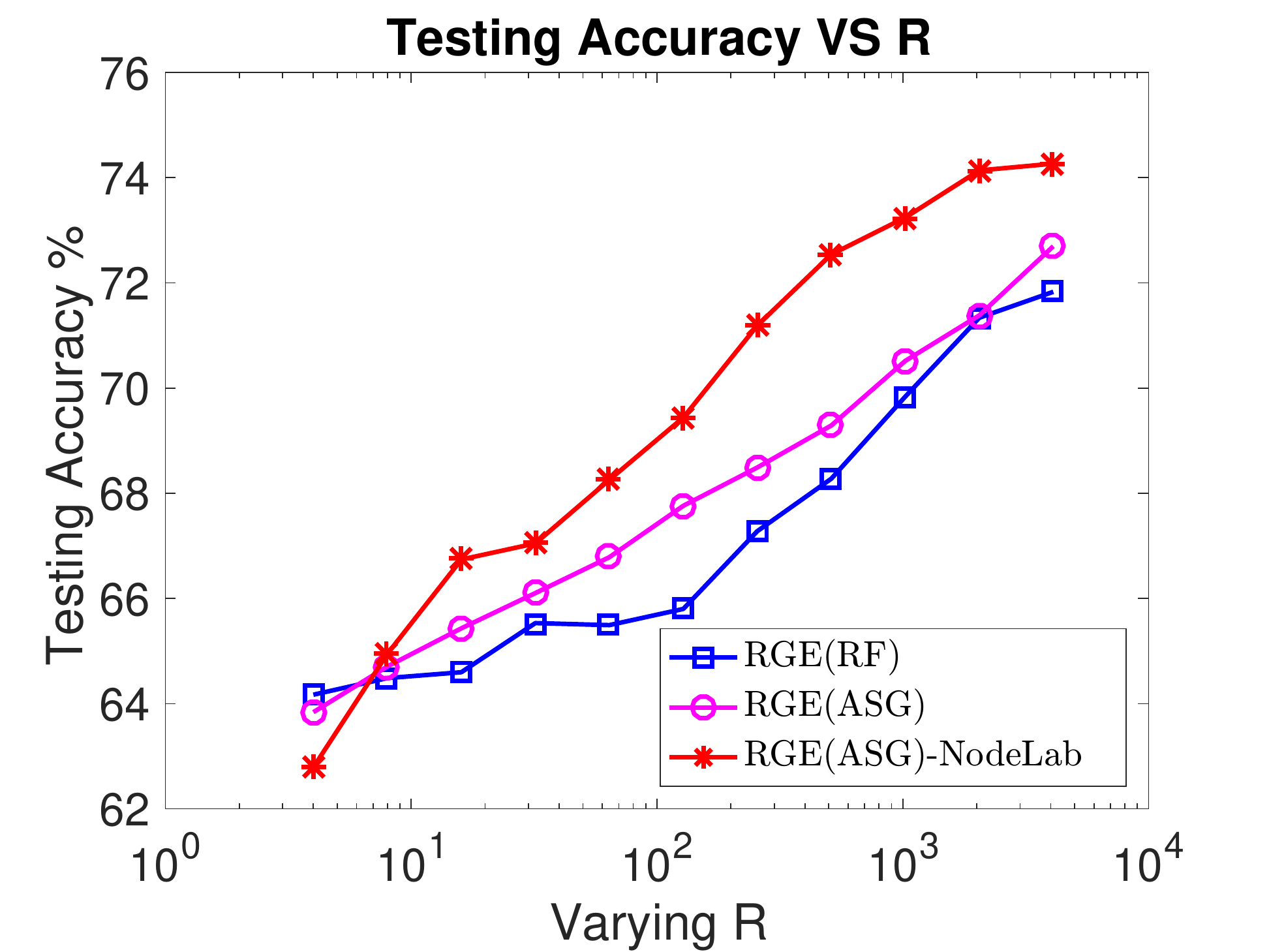}
      \caption{NCI109}
      \label{fig:exptsA_accu_varyingR_NCI109}
    \end{subfigure}
    \begin{subfigure}[b]{0.245\textwidth}
      \includegraphics[width=\textwidth]{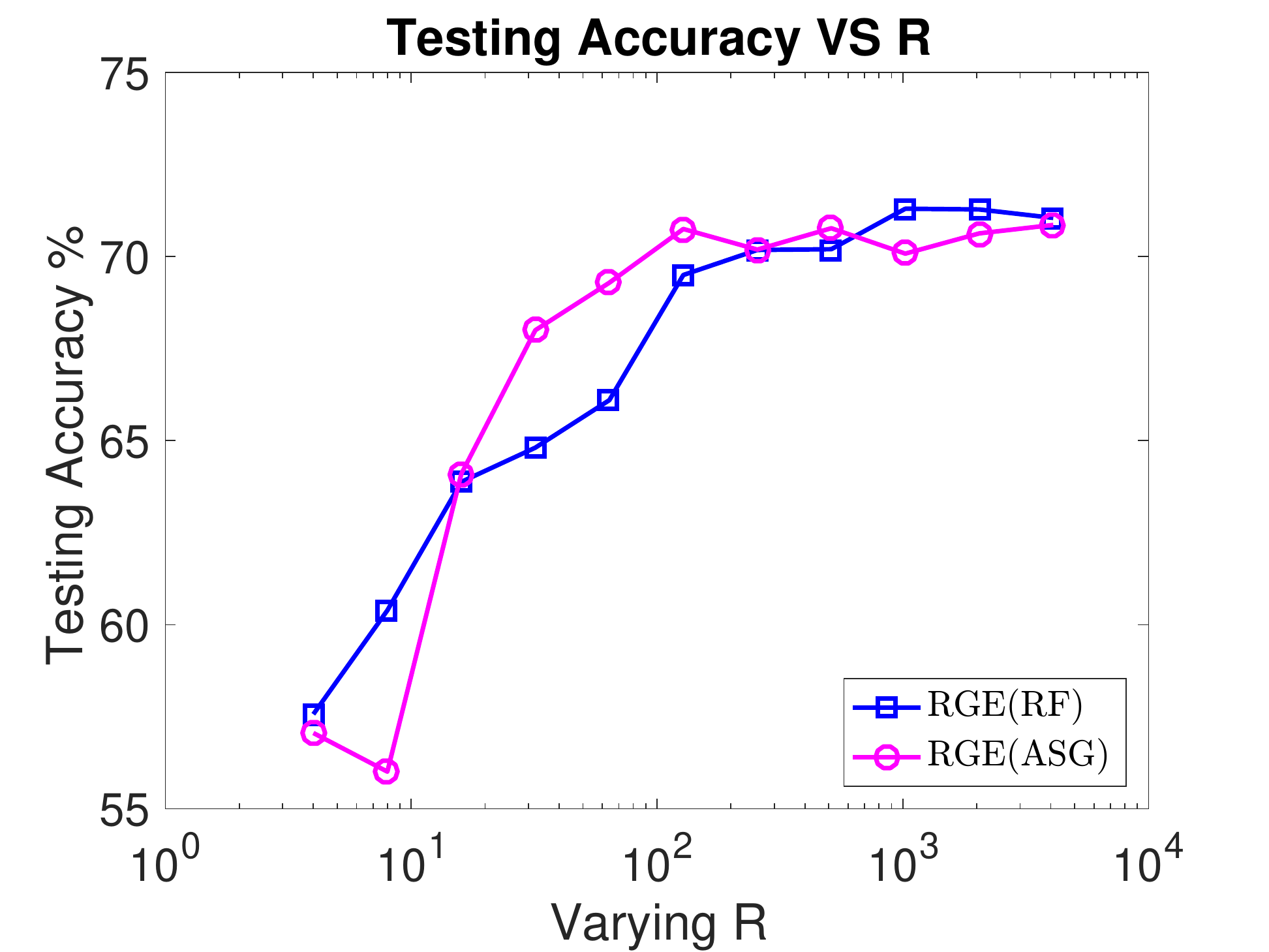}
      \caption{IMDBBINARY}
      \label{fig:exptsA_accu_varyingR_IMDBBINARY}
    \end{subfigure}
    \begin{subfigure}[b]{0.245\textwidth}
      \includegraphics[width=\textwidth]{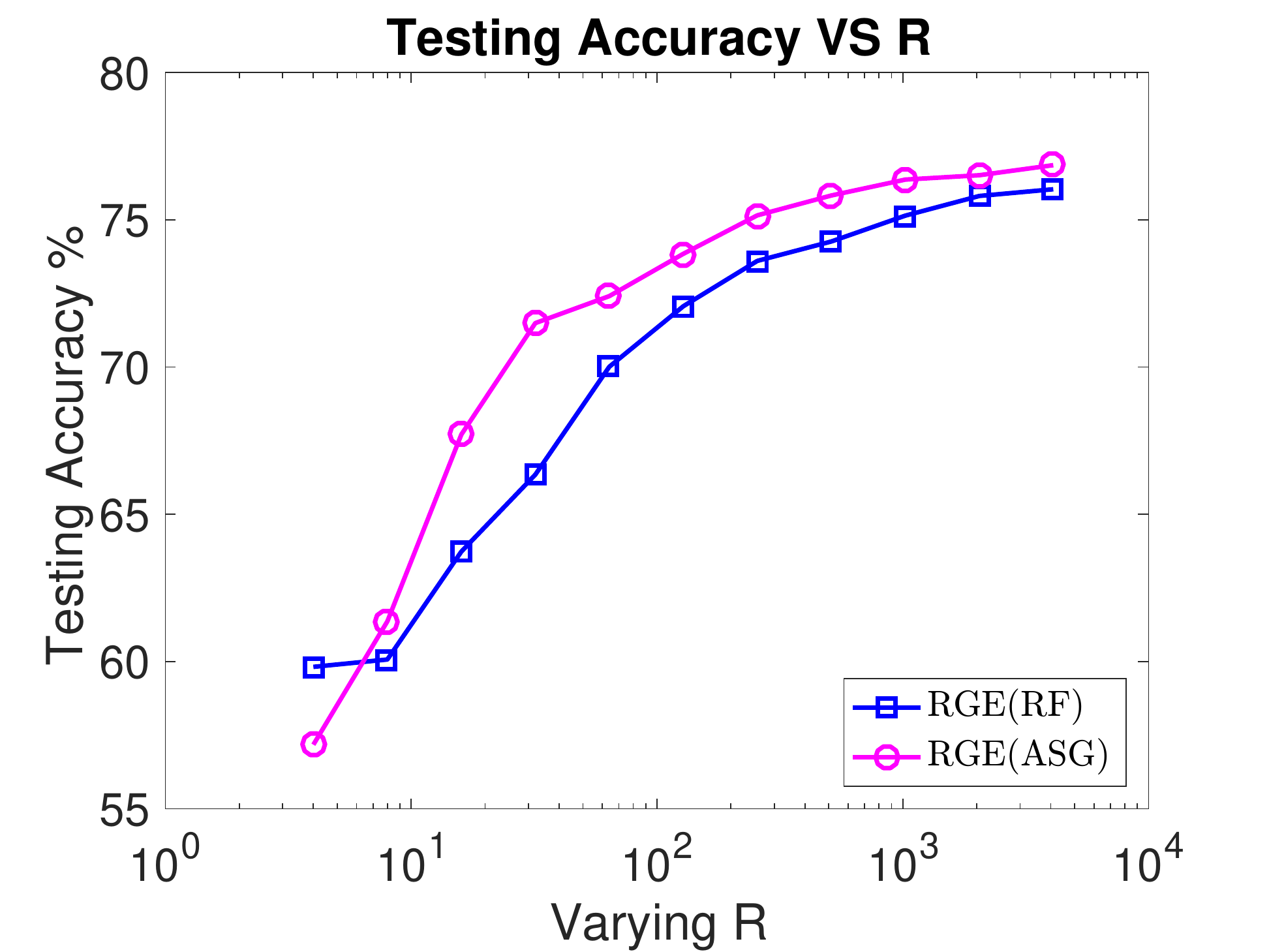}
      \caption{COLLAB}
      \label{fig:exptsA_accu_varyingR_COLLAB}
    \end{subfigure}
    \begin{subfigure}[b]{0.245\textwidth}
      \includegraphics[width=\textwidth]{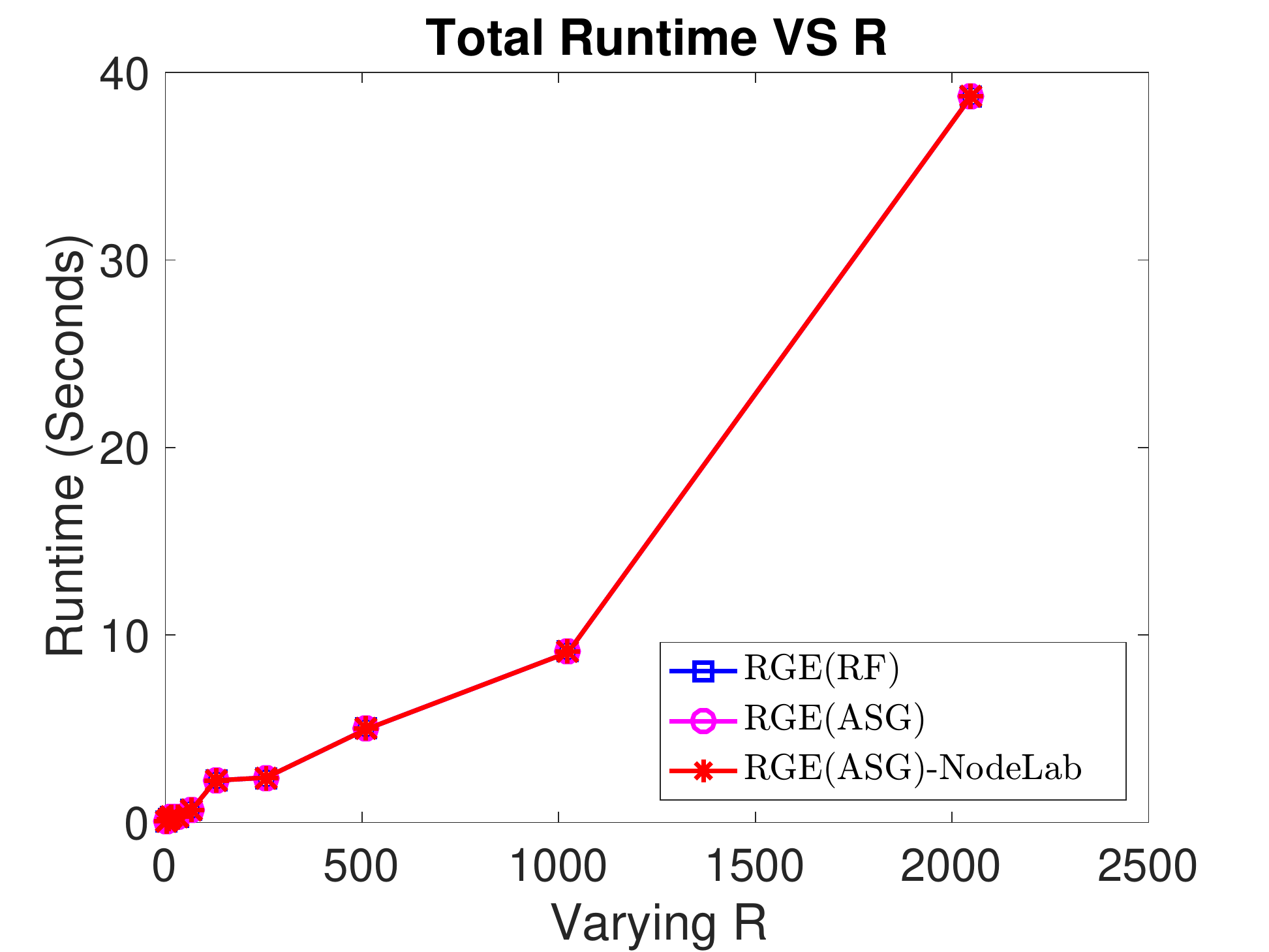}
      \caption{ENZYMES}
      \label{fig:exptsA_time_varyingR_ENZYMES}
    \end{subfigure}
    \begin{subfigure}[b]{0.245\textwidth}
      \includegraphics[width=\textwidth]{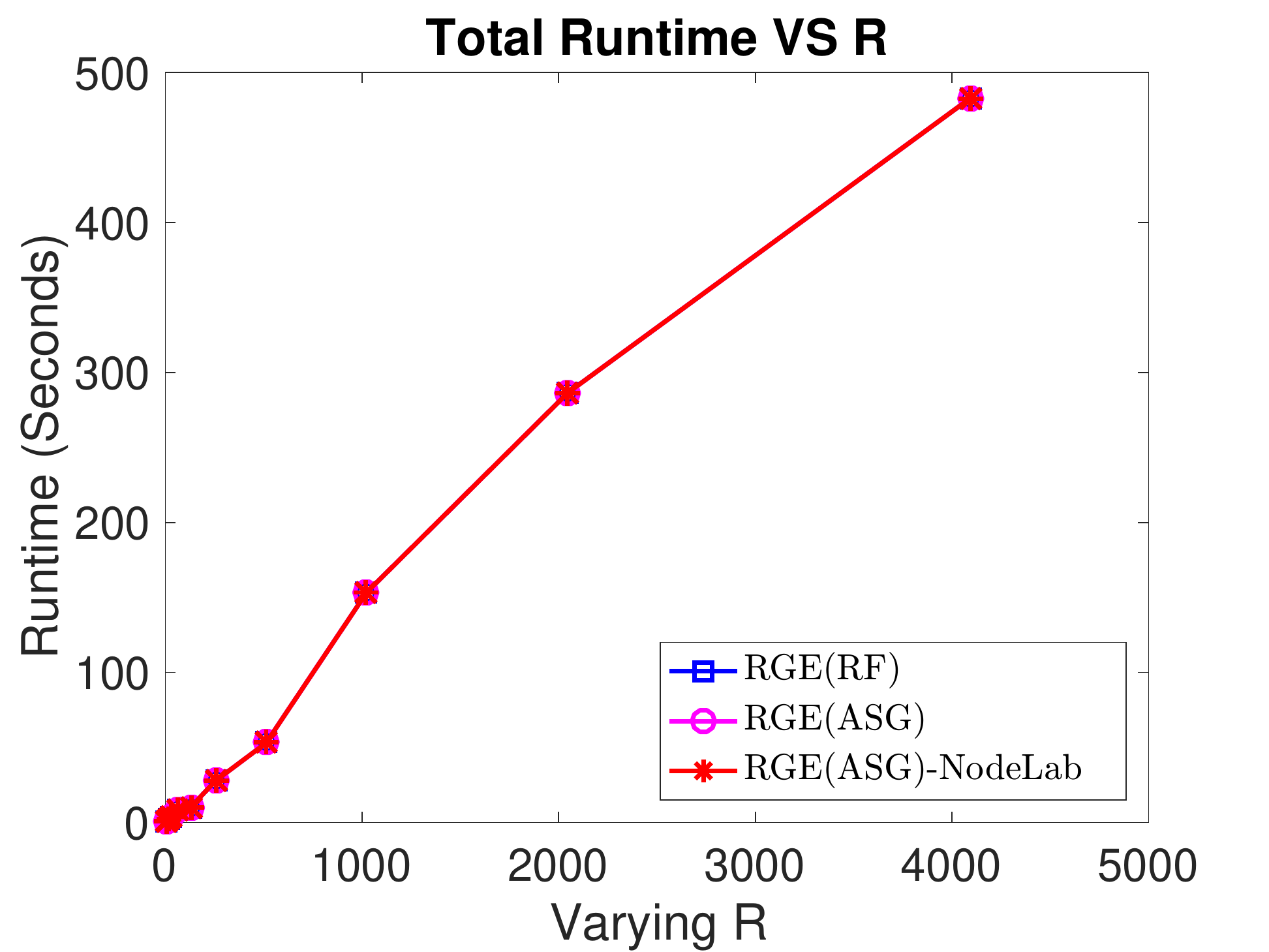}
      \caption{NCI109}
      \label{fig:exptsA_time_varyingR_NCI109}
    \end{subfigure}
    \begin{subfigure}[b]{0.245\textwidth}
      \includegraphics[width=\textwidth]{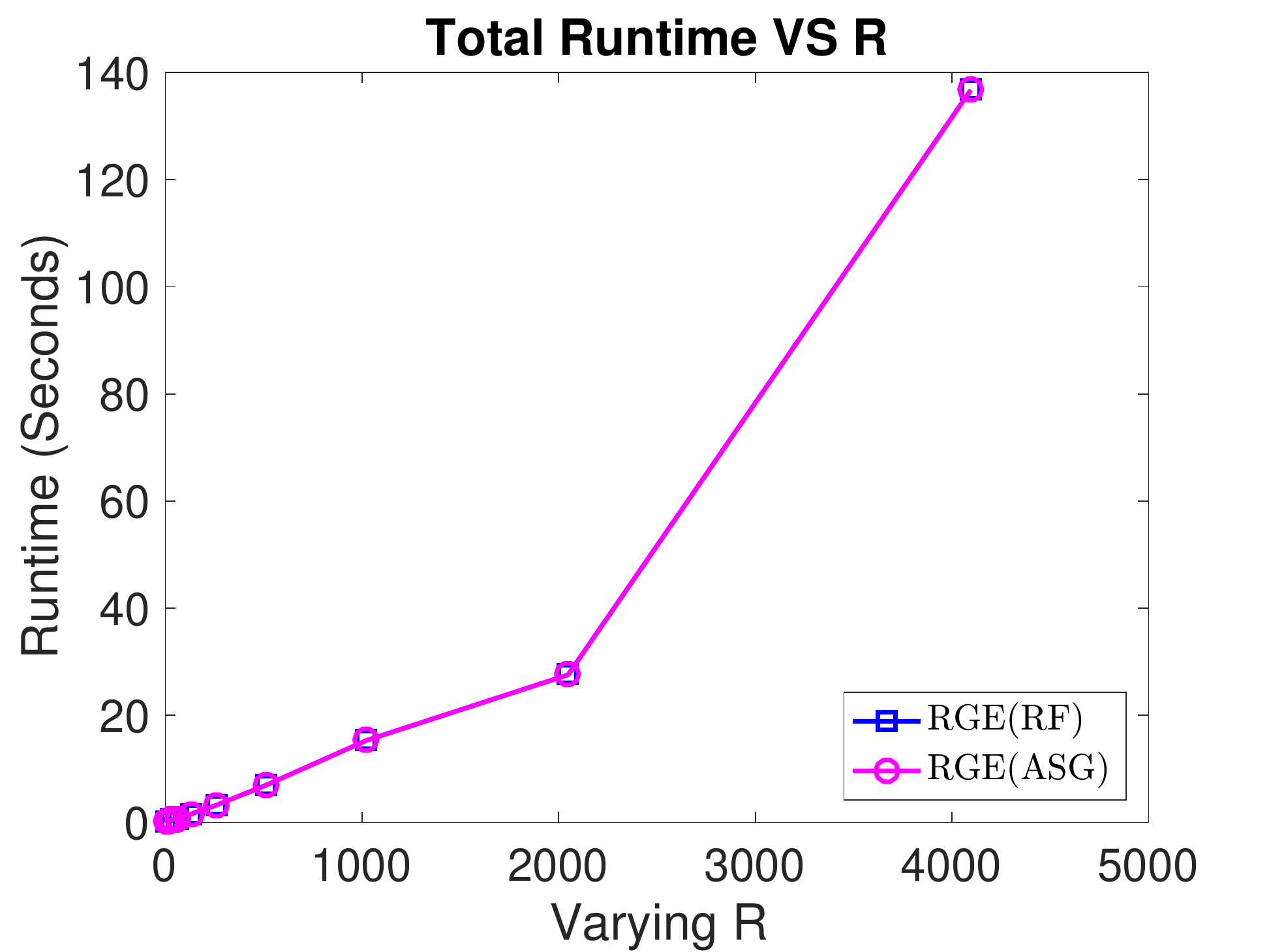}
      \caption{IMDBBINARY}
      \label{fig:exptsA_time_varyingR_IMDBBINARY}
    \end{subfigure}
     \begin{subfigure}[b]{0.245\textwidth}
      \includegraphics[width=\textwidth]{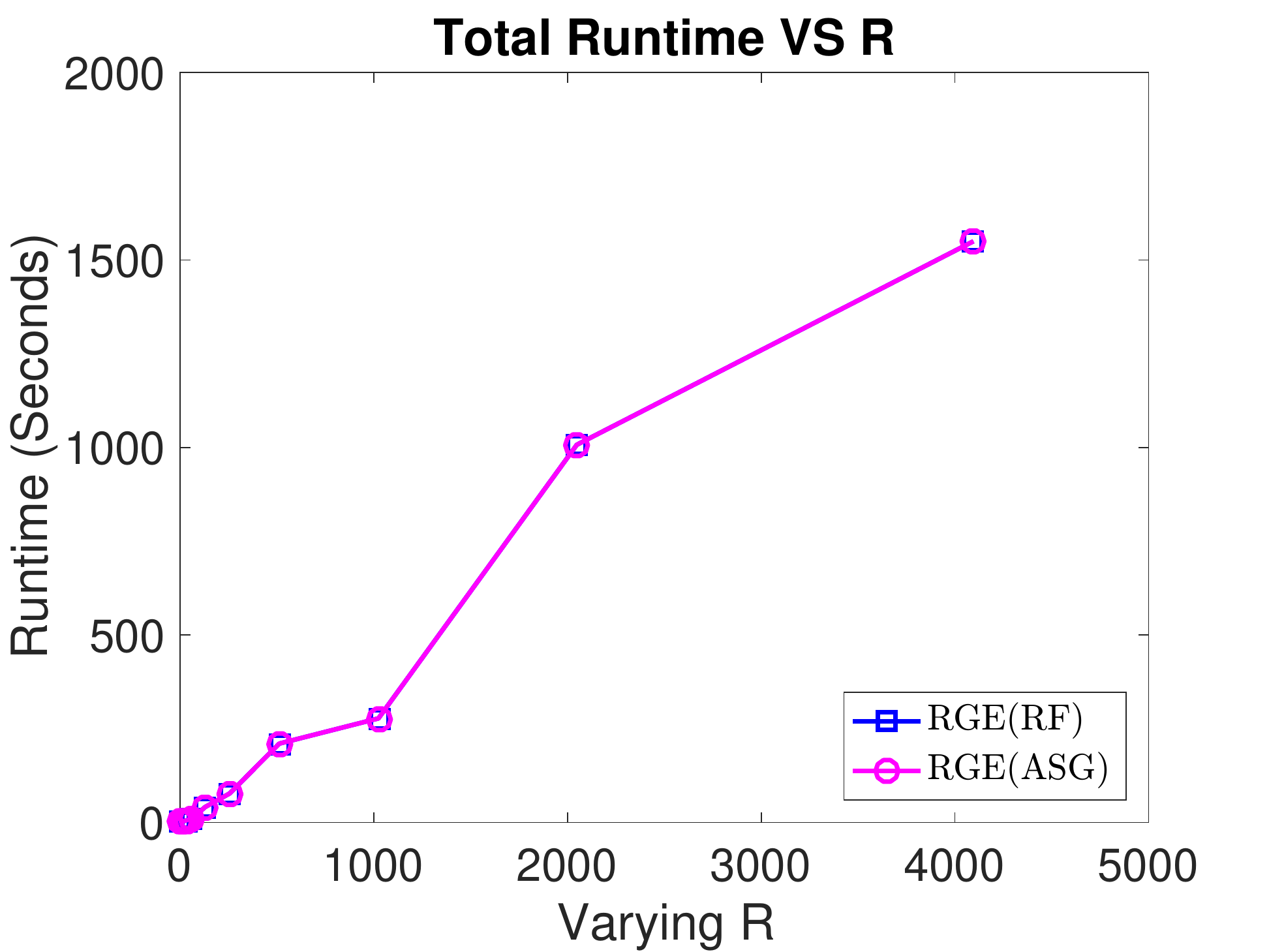}
      \caption{COLLAB}
      \label{fig:exptsA_time_varyingR_COLLAB}
    \end{subfigure}
\vspace{-2mm}
\caption{Test accuracies and runtime of three variants of RGE with and without node labels when varying $R$. }
\label{fig:exptsA_accu_time_varyingR}
\end{figure*}

\begin{figure*}[htbp]
\vspace{-0mm}
   \centering
    \begin{subfigure}[b]{0.40\textwidth}
      \includegraphics[width=\textwidth]{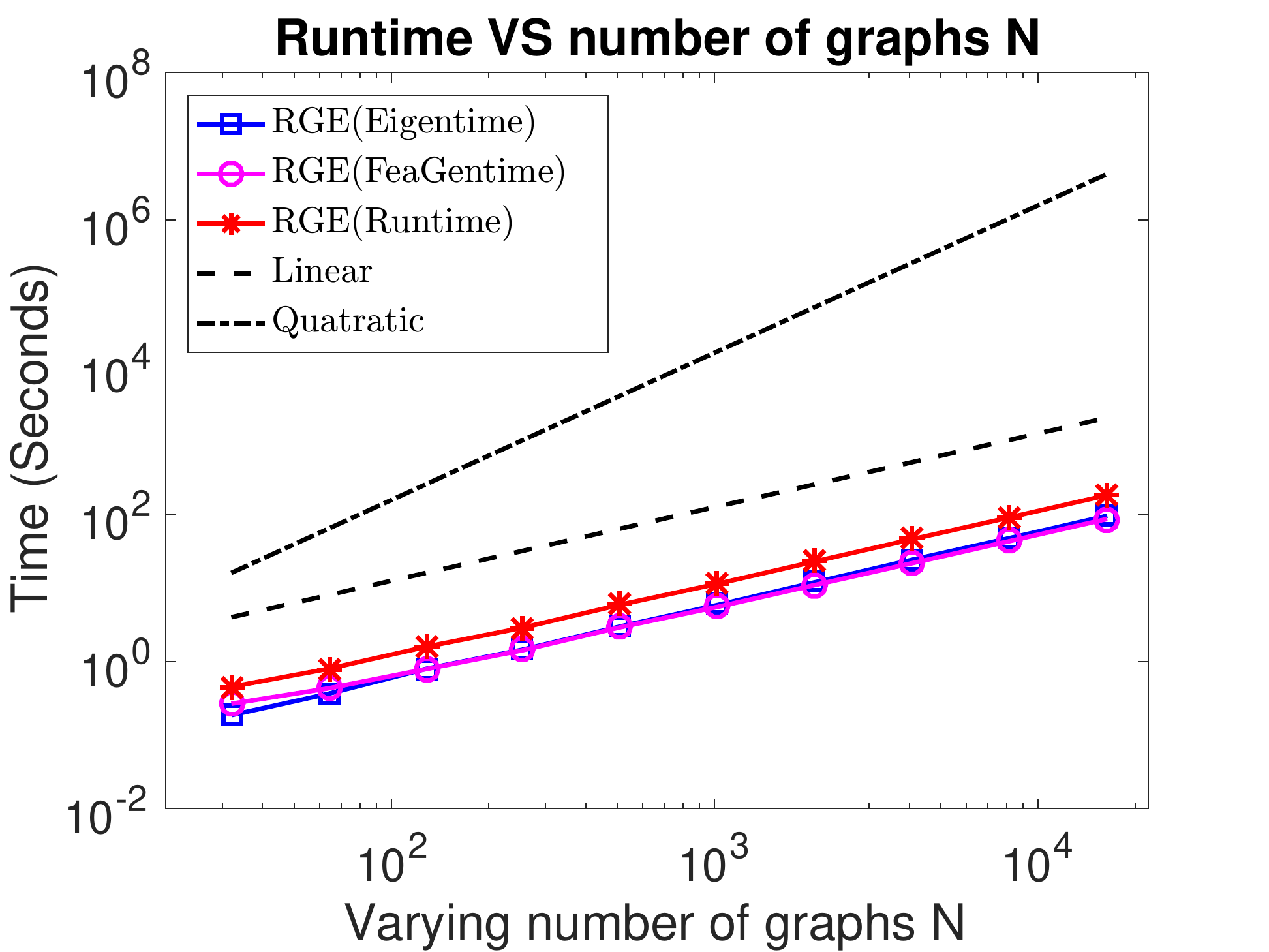}
      \caption{Number of graphs N}
      \label{fig:exptsB_time_varyingNumGs}
    \end{subfigure}
  \begin{subfigure}[b]{0.40\textwidth}
      \includegraphics[width=\textwidth]{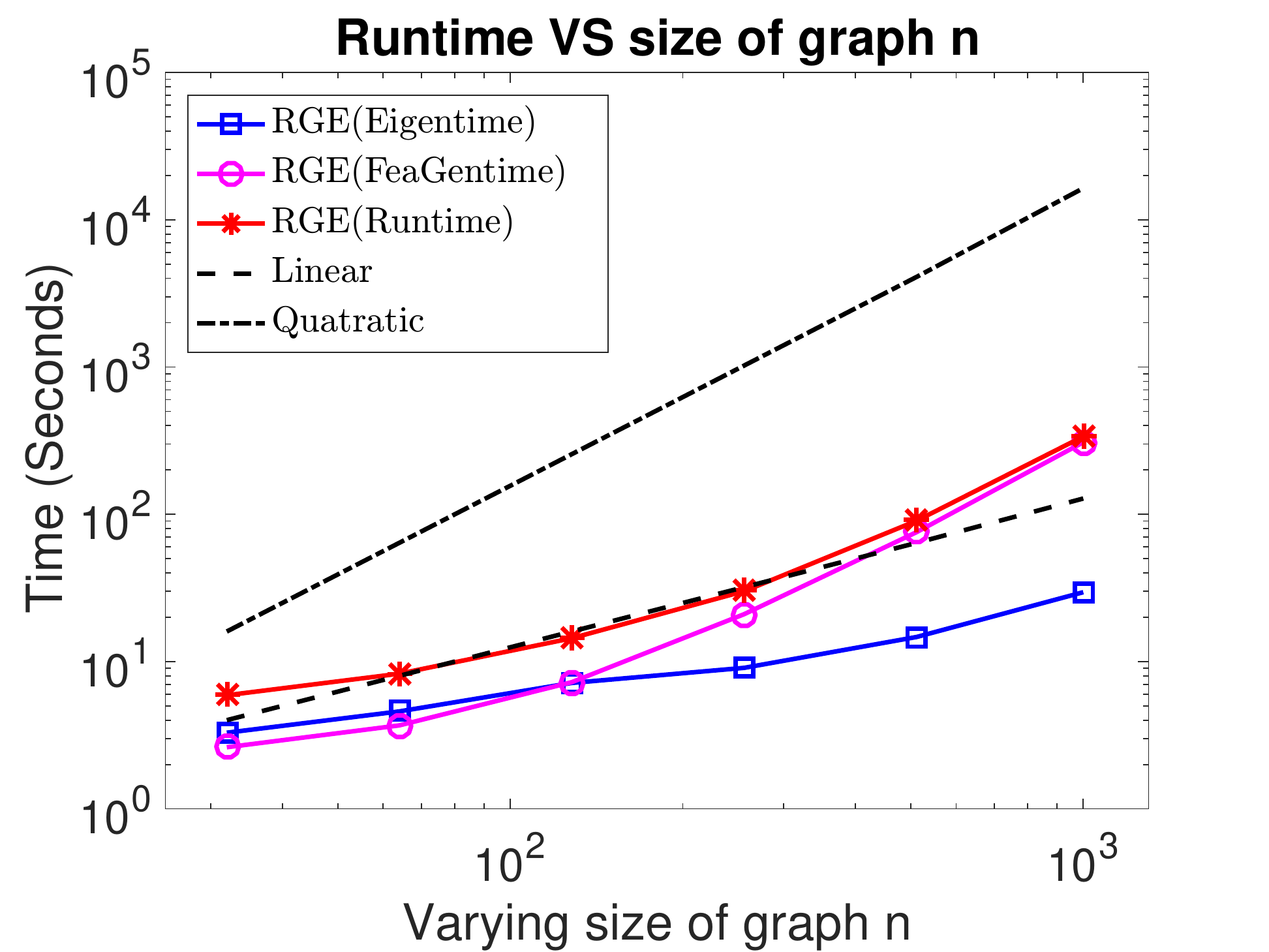}
      \caption{Size of graph n}
      \label{fig:exptsB_time_varyingNumNs}
    \end{subfigure}
\vspace{-2mm}
\caption{Runtime for computing node embeddings and RGE graph embeddings, and overall runtime when varying number of graphs N and size of graph n. (Default values: number of graphs $N=1000$, graph size $n=100$, edge size $m=200$). Linear and quadratic complexity are also plotted for easy comparison.}
\vspace{-0mm}
\label{fig:exptsB_time_varyingNumGs_NumNs}
\end{figure*}

\textbf{Setup.} Since RGE is a graph embedding, we directly employ a linear SVM implemented in LIBLINEAR \cite{fan2008liblinear} since it can faithfully separate the effectiveness of our feature representation from the power of the nonlinear learning solvers. Following the convention of the graph kernel literature, we perform 10-fold cross-validation, using 9 folds for training and 1 for testing, and repeat the whole experiments ten times (thus 100 runs per dataset) and report the average prediction accuracies and standard deviations. The ranges of hyperparameters $\gamma$ and $D\_max$ are [1e-3 1e-2 1e-1 1 10] and [3:3:30], respectively. All parameters of the SVM and hyperparameters of our method were optimized only on the training dataset. 
%The node embedding size is set to either 6 or 8 but always be the same number for all variants of RGE on the same datasets. 
To eliminate the random effects, we repeat the whole experiments ten times and report the average prediction accuracies and standard deviations. 
For all baselines we have taken the best reported number from their papers. Since EMD is the closet method to ours, we execute both methods under the same setting for fair comparison and report both accuracy and computational time. 
%Since GRE, OA, EMD, and PM are essentially built on the same node embeddings from the adjacency matrices, we take the number of OA-{$E_\lambda$}(A) in \cite{johansson2015learning} for a fair comparison. 

% \subsection{Accuracy and Runtime of RGE varying R}
\textbf{Impacts of $R$ on Accuracy and Runtime of RGE.} 
% \textbf{Scalability of RGE when varying $R$ random graphs.}
We conducted experiments investigating the convergence behavior and the scalability of three variants of RGE with or without using node labels when increasing the number $R$ of random graphs.
The hyperparameter $D$ is obtained from the previous cross-validations on the training set. 
% Depending on the size of graph on each dataset, we set $R$ in the range starting from 4 and ending with a number $R$ just satisfying $R = 2^k > n$. 
We report both testing accuracy and runtime when increasing graph embedding size $R$. 
As shown in Fig. \ref{fig:exptsA_accu_time_varyingR}, all variants of RGE converge very rapidly when increasing R from a small number (R = 4) to relatively large number ($R = 2^k > n$). 
%Interestingly, from the trend of the testing accuracy of each variant, it seems that they still have room to further improve the performance to match the optimal performance of exact kernel with even larger $R$. 
This confirms our analysis in Theorem \ref{thm:RF} that the RGE approximation can guarantee rapid convergence to the exact kernel. The second observation is that RGE exhibits quasi-linear scalability with respect to $R$, as predicted by our computational analysis. This is particularly important for large scale graph data since most graph kernels have quadratic complexity in the number of graphs and/or in the size of graphs.

\begin{table*}[htbp]
\centering
\caption{Comparison of classification accuracy against graph kernel methods without node labels.} 
\label{tb:comp_graphs_wo_node_label}
% \small
\newcommand{\Bd}[1]{\textbf{#1}}
\vspace{-3mm}
\begin{center}
    \begin{tabular}{ c c c c c c }
    \hline
    \multicolumn{1}{c}{Datasets}
    & \multicolumn{1}{c}{MUTAG} 
    & \multicolumn{1}{c}{PTC-MR}
    & \multicolumn{1}{c}{ENZYMES} 
    & \multicolumn{1}{c}{NCI1} 
    & \multicolumn{1}{c}{NCI019} \\ \hline 
% 	\multicolumn{1}{c}{Methods} & Accu & Accu & Accu & Accu & Accu  \\ \hline
    RGE(RF)  & \Bd{86.33 {\tiny$\pm$ 1.39}}(1s) & 59.82 {\tiny$\pm$ 1.42}(1s) & 35.98 {\tiny$\pm$ 0.89}(38s)  & \Bd{74.70 {\tiny$\pm$ 0.56}}(727s) & 72.50 {\tiny$\pm$ 0.32}(865s) \\
    RGE(ASG)  & 85.56 {\tiny$\pm$ 0.91}(2s) & \Bd{59.97 {\tiny$\pm$ 1.65}} (1s) & \Bd{38.52 {\tiny$\pm$ 0.91}}(18s) & 74.30 {\tiny$\pm$ 0.45}(579s) & \Bd{72.70 {\tiny$\pm$ 0.42}}(572s)  \\ \hline
    EMD  & 84.66 {\tiny $\pm$ 2.69} (7s) & 57.65 {\tiny$\pm$ 0.59}  (46s) & 35.45 {\tiny$\pm$ 0.93}  (216s) & 72.65 {\tiny$\pm$ 0.34}  (8359s) & 70.84 {\tiny$\pm$ 0.18} (8281s) \\
    PM & 83.83 {\tiny$\pm$ 2.86} & 59.41 {\tiny$\pm$ 0.68} & 28.17 {\tiny$\pm$ 0.37} & 69.73 {\tiny$\pm$ 0.11} & 68.37 {\tiny$\pm$ 0.14} \\
    Lo-$\theta$ & 82.58 {\tiny$\pm$ 0.79} & 55.21 {\tiny$\pm$ 0.72} & 26.5 {\tiny$\pm$ 0.54} & 62.28 {\tiny$\pm$ 0.34} & 62.52 {\tiny$\pm$ 0.29} \\
    OA-{$E_\lambda$}(A) & 79.89 {\tiny$\pm$ 0.98} & 56.77 {\tiny$\pm$ 0.85} & 36.12 {\tiny$\pm$ 0.81} & 67.99 {\tiny$\pm$ 0.28} & 67.14 {\tiny$\pm$ 0.26}\\
	RW & 77.78 {\tiny$\pm$ 0.98} & 56.18 {\tiny$\pm$ 1.12} & 20.17 {\tiny$\pm$ 0.83} & 56.89 {\tiny$\pm$ 0.34} & 56.13 {\tiny$\pm$ 0.31} \\
	GL & 66.11 {\tiny$\pm$ 1.31} & 57.05 {\tiny$\pm$ 0.83} & 18.16 {\tiny$\pm$ 0.47} & 47.37 {\tiny$\pm$ 0.15} & 48.39 {\tiny$\pm$ 0.18} \\
	SP  & 82.22 {\tiny$\pm$ 1.14} & 56.18 {\tiny$\pm$ 0.56} & 28.17 {\tiny$\pm$ 0.64} & 62.02 {\tiny$\pm$ 0.17} & 61.41 {\tiny$\pm$ 0.32} \\
    \hline
    \end{tabular}   
\end{center}
\vspace{-2mm}
\end{table*}

\begin{table*}[htbp]
\centering
\caption{Comparison of classification accuracy against graph kernel methods with node labels or WL technique.} 
\label{tb:comp_graphs_node_label}
% \small
\newcommand{\Bd}[1]{\textbf{#1}}
\vspace{-3mm}
\begin{center}
    \begin{tabular}{ c c c c c c }
    \hline
    \multicolumn{1}{c}{Datasets}
    & \multicolumn{1}{c}{PTC-MR}
    & \multicolumn{1}{c}{ENZYMES} 
    & \multicolumn{1}{c}{PROTEINS} 
    & \multicolumn{1}{c}{NCI1} 
    & \multicolumn{1}{c}{NCI019} \\ \hline 
% 	\multicolumn{1}{c}{Methods} & Accu & Accu & Accu & Accu & Accu  \\ \hline
    RGE(ASG)  & \Bd{61.5 {\tiny$\pm$ 2.34}}(1s) & \Bd{48.27 {\tiny$\pm$ 0.99}}(28s) & 75.98 {\tiny$\pm$ 0.71}(20s)  & \Bd{76.46 {\tiny$\pm$ 0.45}}(379s) & \Bd{74.42 {\tiny$\pm$ 0.30}}(526s)  \\ \hline
    EMD  & 57.67 {\tiny$\pm$ 2.11} (42s) & 42.85 {\tiny$\pm$ 0.72} (296s) & \Bd{76.03 {\tiny$\pm$ 0.28}} (1936s) & 75.89 {\tiny$\pm$ 0.16} (7942s) & 73.63 {\tiny$\pm$ 0.33} (8073s) \\
    PM & 60.38 {\tiny$\pm$ 0.86} & 40.33 {\tiny$\pm$ 0.34} & 74.39 {\tiny$\pm$ 0.45}  & 72.91 {\tiny$\pm$ 0.53} & 71.97 {\tiny$\pm$ 0.15} \\
    OA-{$E_\lambda$}(A) & 58.76 {\tiny$\pm$ 0.92} & 43.56 {\tiny$\pm$ 0.66} & --- & 69.83 {\tiny$\pm$ 0.30} & 68.96 {\tiny$\pm$ 0.35} \\
    V-OA & 56.4 {\tiny$\pm$ 1.8} & 35.1 {\tiny$\pm$ 1.1} & 73.8 {\tiny$\pm$ 0.5} & 65.6 {\tiny$\pm$ 0.4} & 65.1 {\tiny$\pm$ 0.4}\\
	RW  & 57.06 {\tiny$\pm$ 0.86} & 19.33 {\tiny$\pm$ 0.62} & 71.67 {\tiny$\pm$ 0.78} & 63.34 {\tiny$\pm$ 0.27} & 63.51 {\tiny$\pm$ 0.18} \\
    GL & 59.41 {\tiny$\pm$ 0.94} & 32.70 {\tiny$\pm$ 1.20} &  71.63 {\tiny$\pm$ 0.33} & 66.00 {\tiny$\pm$ 0.07} & 66.59 {\tiny$\pm$ 0.08} \\
	SP & 60.00 {\tiny$\pm$ 0.72} & 41.68 {\tiny $\pm$ 1.79} & 73.32 {\tiny$\pm$ 0.45} & 73.47 {\tiny$\pm$ 0.11} & 73.07 {\tiny$\pm$ 0.11} \\
    \hline
    WL-RGE(ASG)  & \Bd{62.20 {\tiny$\pm$ 1.67}}(1s) & 57.97 {\tiny$\pm$ 1.16}(38s) & \Bd{76.63 {\tiny$\pm$ 0.82}}(30s)  & \Bd{85.85 {\tiny$\pm$ 0.42}}(401s) & \Bd{85.32 {\tiny$\pm$ 0.29}}(798s)  \\ \hline
    WL-ST & 57.64 {\tiny$\pm$ 0.68} & 52.22 {\tiny$\pm$ 0.71} & 72.92 {\tiny$\pm$ 0.67} & 82.19 {\tiny$\pm$ 0.18} & 82.46 {\tiny$\pm$ 0.24}\\ 
    WL-SP & 56.76 {\tiny$\pm$ 0.78} & \Bd{59.05 {\tiny$\pm$ 1.05}} & 74.49 {\tiny$\pm$ 0.74} & 84.55 {\tiny$\pm$ 0.36} & 83.53 {\tiny$\pm$ 0.30}\\
    WL-OA-{$E_\lambda$}(A) & 59.72 {\tiny$\pm$ 1.10} & 53.76 {\tiny$\pm$ 0.82} & --- & 84.75 {\tiny$\pm$ 0.21} & 84.23 {\tiny$\pm$ 0.19}\\
    \hline
    \end{tabular}   
\end{center}
\vspace{-2mm}
\end{table*}

\begin{table*}[htbp]
\centering
\caption{Comparison of classification accuracy against recent deep learning models on graphs.} 
\label{tb:comp_DLApproaches}
% \small
\newcommand{\Bd}[1]{\textbf{#1}}
\vspace{-3mm}
\begin{center}
    \begin{tabular}{ c c c c c c c }
    \hline
    \multicolumn{1}{c}{Datasets}
    & \multicolumn{1}{c}{PTC-MR}
    & \multicolumn{1}{c}{PROTEINS}
    & \multicolumn{1}{c}{NCI1} 
    & \multicolumn{1}{c}{IMDB-B} 
    & \multicolumn{1}{c}{IMDB-M}
    & \multicolumn{1}{c}{COLLAB} \\ \hline 
% 	\multicolumn{1}{c}{Methods} & Accu & Accu & Accu & Accu & Accu  \\ \hline
    (WL-)RGE(ASG)  & 62.20 {\tiny$\pm$ 1.67} & \Bd{76.63 {\tiny$\pm$ 0.82}} & \Bd{85.85 {\tiny$\pm$ 0.42}} & \Bd{71.48 {\tiny$\pm$ 1.01}} & 47.26 {\tiny$\pm$ 0.89} & \Bd{76.85 {\tiny$\pm$ 0.34}} \\ \hline
    DGCNN  & 58.59 {\tiny$\pm$ 2.47} & 75.54 {\tiny$\pm$ 0.94} & 74.44 {\tiny$\pm$ 0.47} & 70.03 {\tiny$\pm$ 0.86} & \Bd{47.83 {\tiny$\pm$ 0.85}} & 73.76 {\tiny$\pm$ 0.49} \\
    PSCN & \Bd{62.30 {\tiny$\pm$ 5.70}} & 75.00 {\tiny$\pm$ 2.51} & 76.34 {\tiny$\pm$ 1.68} & 71.00 {\tiny$\pm$ 2.29} & 45.23 {\tiny$\pm$ 2.84} & 72.60 {\tiny$\pm$ 2.15} \\
    DCNN & 56.6 {\tiny$\pm$ 1.20} & 61.29 {\tiny$\pm$ 1.60}  & 56.61 {\tiny$\pm$ 1.04} & 49.06 {\tiny$\pm$ 1.37} & 33.49 {\tiny$\pm$ 1.42} & 52.11 {\tiny$\pm$ 0.53} \\
    DGK & 57.32 {\tiny$\pm$ 1.13} & 71.68 {\tiny$\pm$ 0.50} & 62.48 {\tiny$\pm$0.25} & 66.96 {\tiny$\pm$ 0.56} & 44.55 {\tiny$\pm$ 0.52} & 73.09 {\tiny$\pm$ 0.25} \\
    \hline
    \end{tabular}   
\end{center}
\vspace{-2mm}
\end{table*}

% \subsection{Scalability of RGE varying $N$ graphs and $n$ nodes}
\textbf{Scalability of RGE varying $N$ graphs and $n$ nodes.} 
We further assess the scalability of RGE when varying number of graphs $N$ and size of graph $n$ for randomly generated graphs. 
We change the number of graphs in the range of $N = [8 \ 16384]$ and the size of graph in the range of $n = [8 \ 1024]$, respectively. 
When generating random adjacency matrices, we set the number of edges always be twice the number of nodes in a graph. 
%We use the size of node embedding $d = 6$ just like in the previous sections. We set the hyperparameters related to RGE itself are $DMax = 10$ and $R = 128$. 
We report the runtime for computing node embeddings using a state-of-the-art eigensolver \cite{wu2017primme_svds}, generating RGE graph embeddings, and the overall computation of graph classification, accordingly. Fig. \ref{fig:exptsB_time_varyingNumGs_NumNs}(a) shows the linear scalability of RGE when increasing the number of graphs, confirming our complexity analysis in the previous Section. In addition, as shown in Fig. \ref{fig:exptsB_time_varyingNumGs_NumNs}(b), RGE still exhibits linear scalablity in computing eigenvectors but slightly quasi-linear scalablity in RGE generation time and overall time, when increasing the size of graph. 
%When the size of graph $n$ approaches 1000, a relatively large number, the feature generation time for RGE moves from linear complexity towards quasi-linear complexity. 
This is because that even though RGE reduces conventional EMD's complexity from super-cubic $O(n^3log(n)$ to $O(D^2nlog(n)$ (where D is a small constant), the log factor starts to show its impact on computing EMD between raw graphs and small random graphs when $n$ becomes large (e.g. close to 1000). Interestingly, with a state-of-the-art eigensolver, the complexity of computing a few eigenvectors is linearly proportional to the graph size $n$ \cite{wu2017primme_svds}. This is highly desired property of our RGE embeddings, which open the door to large-scale applications of graph kernels for various applications such as social networks analysis and computational biology.

% \subsection{Comparison with All Baselines}
\textbf{Comparison with All Baselines.} Tables \ref{tb:comp_graphs_wo_node_label}, \ref{tb:comp_graphs_node_label}, and \ref{tb:comp_DLApproaches} show that RGE consistently outperforms or matches other state-of-the-art graph kernels and deep learning approaches in terms of classification accuracy. There are several further observations worth making here. 
%First, RGE(ASG) achieves noticeably better performance than RGE(RF) in all cases, which implies that random graphs sampled from data-dependent distribution may better capture the relative importance of each node identified by EMD. 
First, EMD, the closest method to RGE, shows good performance compared to most of other methods but often has significantly worse performance than RGE, highlighting the utility the novel graph kernel design using a feature map of random graphs and the effectiveness of a truly p.d. kernel. Importantly, RGE is also orders of magnitude faster than EMD in all cases, especially for data with a large graph size (like PROTEINS) or large number of graphs (like NCI1 and NCI109).  

Second, the performance of RGE renders clear the importance of considering global properties graphs, and of having a  distance measure able to align contextually-similar but positionally-different nodes, for learning expressive representations of graphs.
In addition, as shown in Table \ref{tb:comp_graphs_node_label}, we observe that all methods (including RGE) gain performance benefits when considering the node label information or utilizing WL iterations based on node labels. With node label information, the gaps between RGE and other methods diminish but still showing very clear advantages of RGE.

Finally, as shown in Table \ref{tb:comp_DLApproaches}, for biological datasets we used the WL-RGE(ASG) to obtain the best performance with WL iteration. For social network datasets, we used the RGE(ASG) without node label since there are no node labels on these datasets. 
Compared to supervised deep learning based approaches, our unsupervised RGE method yet still shows clear advantages, highlighting the importance of aligning the structural roles of each node when comparing two graphs. In contrast, most of deep learning based methods focus on node-level representations instead of graph-level representation (typically using mean-pooling), which cannot take into account these important structural roles of each node in graphs.

%%%%%%%%%%%%%%%%%%%%%%%%%%%%%%%%%%%%%%%%%%%%%%%%%%%%%%%%%%%%%%%%%%%%%%%%%%%%%%%
\section{Conclusion and Future Work}
In this work, we have presented a new family of p.d. and scalable global graph kernels that take into account global properties of graphs. The benefits of RGE are demonstrated by its much higher graph classification accuracy compared with other graph kernels and its (quasi-)linear scalability in terms of the number of graphs and graph size. Several interesting directions for future work are indicated: i) the graph embeddings generated by our technique can be applied and generalized to other learning problems such as graph (subgraph) matching or searching; ii) extensions of the RGE kernel for graphs with continuous node attributes and edge attributes should be explored.

% \clearpage
% \small
%
% The next two lines define the bibliography style to be used, and the bibliography file.
\bibliographystyle{ACM-Reference-Format}
\bibliography{RGE}

% 
% If your work has an appendix, this is the place to put it.
% \appendix

\clearpage
\appendix
\section{Appendix A: Proofs of Lemma \ref{lemma:cover} and Theorem \ref{thm:RF}} \label{appendix: proof of lemma and theory }

\subsection{Proof of Lemma \ref{lemma:cover}}
\begin{proof}
Since the geometric node embedding $\bu_i$ uses the normalized eigenvectors of the Laplacian matrix, we have that $\Vert\bu_i\Vert_2\leq1$, i.e., $\bu_i$ belongs to a unit ball. Therefore, we can find an $\epsilon$-covering $\E_v$ of size $(1+\frac{2}{\epsilon})^{d}$ for the unit ball. Next, we define $\E$ as all the possible sets of $\bv\in\E_v$ of size no larger than $M$. So we have $|\E|= (1+\frac{2}{\epsilon})^{dM}$. For any graph $G=(\bv_j)_{j=1}^n \in\X$, we can find $G_i\in \E$ with also $n$ nodes $(\bu_j)_{j=1}^n$ such that $\|\bu_j-\bv_j\|\leq \epsilon$. Then by the definition of EMD \eqref{EMD}, a solution that assigns each node $\bv_j$ in $G$ to a node $\bu_j$ in $x_i$ would have overall cost less than $\epsilon$, So $\EMD(G,G_i)\leq \epsilon$.
\end{proof}

\subsection{Proof of Proposition \ref{coverrelation}}
\begin{proof}
For any $G_x, G_y\in \mathcal{X}$, we can find $G_{x_k}, G_{y_k}\in \E$, such that
\begin{equation}\label{Proposition:eq1}
\mathrm{EMD}(G_x, G_{x_k})\leq \frac{t}{4\gamma}\quad \mathrm{and} \quad
\mathrm{EMD}(G_y, G_{y_k})\leq \frac{t}{4\gamma}.
\end{equation}
Write $\Delta_R(G_x,G_y)=\Delta_R(G_{x_k}, G_{y_k})+\Delta_R(G_x,G_y)-\Delta_R(G_{x_k}, G_{y_k})$, then we have
\begin{equation}\label{Proposition:eq2}
\begin{aligned}
&\vert\Delta_R(G_x,G_y)\vert\\
\leq&|\Delta_R(G_{x_k}, G_{y_k})|+|\Delta_R(G_x,G_y)-\Delta_R(G_{x_k}, G_{y_k})|\\
\leq&|\Delta_R(G_{x_k}, G_{y_k})|+|\tilde{k}_R(G_{x_k},G_{y_k})-\tilde{k}_R(G_{x},G_{y})|\\
&+|{k}_R(G_{x_k},G_{y_k})-{k}_R(G_{x},G_{y})|
\end{aligned}
\end{equation}
Now we consider the second term. 
\begin{equation}\label{proposition:eq3}
\begin{aligned}
&|\tilde{k}_R(G_{x_k},G_{y_k})-\tilde{k}_R(G_{x},G_{y})|\\
\leq&\frac{1}{R}\sum^R_{i=1}\vert\exp\big(-\gamma\mathrm{EMD}(G_{x_k},G_{\omega_i})-\gamma\mathrm{EMD}(G_{y_k},G_{\omega_i})\big)-\\
&\quad \exp\big(-\gamma\mathrm{EMD}(G_{x},G_{\omega_i})-\gamma\mathrm{EMD}(G_{y},G_{\omega_i})\big)\vert\\
\leq&\frac{1}{R}\sum^R_{i=1}\gamma\vert\mathrm{EMD}(G_{x_k},G_{\omega_i})+\mathrm{EMD}(G_{y_k},G_{\omega_i})\\
&\quad-\mathrm{EMD}(G_{x},G_{\omega_i})-\mathrm{EMD}(G_{y},G_{\omega_i})\vert\\
\leq&\frac{1}{R}\sum^R_{i=1}\gamma\vert \mathrm{EMD}(G_{x_k},G_{\omega_i})-\mathrm{EMD}(G_{x},G_{\omega_i})\vert+\\
&\frac{1}{R}\sum^R_{i=1}\gamma\vert \mathrm{EMD}(G_{y_k},G_{\omega_i})-\mathrm{EMD}(G_{y},G_{\omega_i})\vert\\
\leq&\frac{1}{R}\sum^R_{i=1}\gamma\mathrm{EMD}(G_x, G_{x_k})+\frac{1}{R}\sum^R_{i=1}\gamma\mathrm{EMD}(G_y, G_{y_k})\leq\frac{t}{2}.
\end{aligned}
\end{equation}
Similarly, we can prove that the third term in \eqref{Proposition:eq2} satisfies
\begin{equation}\label{proposition:eq4}
|{k}_R(G_{x_k},G_{y_k})-{k}_R(G_{x},G_{y})|\leq\frac{t}{2}.
\end{equation}
Combining \eqref{proposition:eq3}, \eqref{proposition:eq4}, and the assumption $|\Delta_R(G_{x_k}, G_{y_k})|\leq t$, we obtain the desired result. 
\end{proof}

\subsection{Proof of Theorem \ref{thm:RF}}
\begin{proof}
Based on Proposition \ref{coverrelation}, we have 
\begin{equation}\label{app:them:eq1}
\begin{aligned}
&P\left\{\mathrm{sup}_{G_x, G_y\in\mathcal{X}}\vert\Delta_R(G_x,G_y)\vert\leq 2t\right\}\\
\geq &P\left\{\mathrm{sup}_{G_i, G_j\in\E}\vert\Delta_R(G_x,G_y)\vert\leq t\right\}.
\end{aligned}
\end{equation}
For any $G_i, G_j\in \E$, since $E[\Delta_R(G_i,G_j)]=0$ and $|\Delta_R(G_i,G_j)|\leq 1$, from the Hoeffding inequality, we have
\begin{equation}
P\left\{ |\Delta_R(G_i,G_j)|\geq t \right\} \leq 2 \exp(-Rt^2/2).
\end{equation}
Therefore,
\begin{equation}\label{app:them:eq2}
\begin{aligned}
&P\left\{\mathrm{sup}_{G_i, G_j\in\E}\vert\Delta_R(G_i,G_j)\vert\geq t\right\}\\
\leq&\sum_{G_i, G_j\in\E}P\{\vert\Delta_R(G_i,G_j)\vert\geq t\}\\
\leq &2\vert\E\vert^2\exp(-Rt^2/2)\leq2(1+\frac{8\gamma}{t})^{2dM}\exp(-Rt^2/2).
\end{aligned}
\end{equation}
Combining \eqref{app:them:eq1} and \eqref{app:them:eq2}, and setting $t=\frac{\epsilon}{2}$, we obtain the desired result.
\end{proof}

\section{Appendix B: Additional Experimental Results}
\label{app:Additional Experimental Results}

\begin{table*}[htbp]
\centering
\small
\vspace{-4mm}
\caption{Properties of the datasets. } 
\label{tb:info of datasets}
\vspace{0mm}
\begin{center}
    \begin{tabular}{ c c c c c c c c c c}
    \hline
    Dataset & MUTAG & PTC & ENZYMES & PROTEINS & NCI1 & NCI109 & IMDB-B & IMDB-M & COLLAB\\ \hline 
    Max \# Nodes & 28 & 109 & 126 & 620 & 111 & 111 & 136 & 89 & 492 \\
    Min \# Nodes & 10 & 2 & 2 &4 & 3 & 4 & 12 & 7 & 32 \\
    Ave \# Nodes & 17.9 & 25.6 & 32.6 & 39.05 & 29.9 & 29.7 & 19.77 & 13.0 & 74.49 \\\hline
    Max \# Edges & 33 & 108 & 149 & 1049 & 119 & 119 & 1249 & 1467 & 40119\\ 
    Min \# Edges & 10 & 1 & 1 & 5 & 2 & 3 & 26 & 12 & 60 \\  
    Ave \# Edges & 19.8 & 26.0 & 62.1 & 72.81 & 32.3 & 32.1 & 96.53 & 65.93 & 2457.34  \\\hline
    \# Graph & 188 & 344 & 600 & 1113 & 4110 & 4127 & 1000 & 1500 & 5000 \\
    \# Graph Labels & 2 & 2 & 6 & 2 & 2 & 2 & 2 & 3 & 3 \\
    \# Node Labels & 7 & 19 & 3 & 3 & 37 & 38 & --- & --- & --- \\ \hline
    \end{tabular}
\end{center}
\vspace{-2mm}
\end{table*}

\textbf{General Setup.} We perform experiments to demonstrate the effectiveness and efficiency of the proposed method, and compare against total 12 graph kernels and deep graph neural networks on 9 benchmark datasets (as shown in Table \ref{tb:info of datasets}) \footnote{http://members.cbio.mines-paristech.fr/~nshervashidze/code/} that is widely used for testing the performance of graph kernels. We implement our method in Matlab and utilize C-MEX function \footnote{http://ai.stanford.edu/~rubner/emd/default.htm} for the computationally expensive component of EMD. 
To accelerate the computation, we use multithreading with total 12 threads in all experiments. 
All computations were carried out on a DELL dual socket system with Intel Xeon processors 272 at 2.93GHz for a total of 16 cores and 250 GB of memory, running the SUSE Linux operating system.

\subsection{Additional Results and Discussions on Accuracy and Runtime of RGE Varying $R$}
\label{app:Additional Results and Discussions on Accuracy and Runtime of RGE Varying $R$}

\textbf{Setup.} We now conduct experiments to investigate the behavior of three variants of RGE with or without using node labels by varying the number $R$ of random graphs. 
The hyperparameter $D$ is obtained from the previous cross-validations on the training set. 
Depending on the size of graph on each dataset, we set $R$ in the range starting from 4 and ending with a number $R$ just satisfying $R = 2^k > n$. 
We report both testing accuracy and runtime when increasing graph embedding size $R$. 
%Fig. \ref{fig:exptsA_accu_time_varyingR} and \ref{fig:exptsA_accu_time_varyingR} show how the testing accuracy and runtime changes when increasing $R$. We can see that all variants of RGE converge very fast when increasing R from a small number (R = 4) to relatively large number. Interestingly, from the trend of the testing accuracy of each variant, it seems that they still have room to further improve the performance to match the optimal performance of exact kernel with even larger $R$. This confirms our analysis in Theory 1 that the RGE approximation can guarantee the fast convergence to the exact kernel. The second observation is that RGE indeed has linear complexity with $R$ coincided with our computational analysis in Section \ref{sec:random graph embeddings}. This is particularly important for large size graphs since indefinite kernel with EMD \cite{nikolentzos2017matching} has quasi-cubic complexity in terms of the size of graphs and quadratic complexity in terms of the number of total graphs.  

\subsection{Additional Results and Discussions on Scalability of RGE varying $N$ graphs and $n$ nodes}
\label{app:Additional Results and Discussions on Scalability of RGE varying $N$ graphs and $n$ nodes}

\textbf{Setup.} We assess the scalability of RGE when varying number of graphs $N$ and the size of a graph $n$ on randomly generated graphs. We change the number of graphs in the range of $N = [8 \ 16384]$ and the size of graph in the range of $n = [8 \ 1024]$, respectively. When generating random adjacency matrices, we set the number of edges always be twice the number of nodes in a graph. We use the size of node embedding $d = 6$ just like in the previous sections. We set the hyperparameters related to RGE itself are $DMax = 10$ and $R = 128$. We report the runtime for computing node embeddings using state-of-the-art eigensolver \cite{stathopoulos2010primme,wu2017primme_svds} and RGE graph embeddings, and the overall runtime, respectively. 

% \textbf{Results.} As shown in Fig. \ref{fig:exptsB_time_varyingNumGs_NumNs}, we have two important observations about the scalability of RGE. First, Fig. \ref{fig:exptsB_time_varyingNumGs} clearly shows the linear scalability of RGE in terms of the number of graphs, which confirms our computational analysis in in Section \ref{sec:random graph embeddings}. As shown in Fig. \ref{fig:exptsB_time_varyingNumNs}, when the size of graph $n$ increases close to a relatively large number 1000, the feature generation time for RGE graph embedding starts turning from linear complexity towards quadratic complexity. Recall that RGE reduces conventional EMD from quasi-cubic complexity to $O(D^2nlog(n)$, where D is a small constant. We guess that it is because the log factor starts to show its impact on the computation of EMD between raw graphs and small random graphs. Interestingly, with state-of-the-art eigensolver, it is indeed that the complexity of computing a few of eigenvectors is linearly proportional to the graph size $n$ \cite{wu2017primme_svds}. 

\subsection{Additional Results and Discussions on Comparisons Against All Baselines}
\label{app:Additional Results and Discussions on Comparisons Against All Baselines}

\textbf{Setup.} Since RGE is a graph embedding, we directly employ a linear SVM implemented in LIBLIBNEAR \cite{fan2008liblinear} since it can faithfully examine the effectiveness of our feature representation from the power of the nonlinear learning solvers. Following the convention in the graph kernel literature, we perform 10-fold cross-validation, using 9 folds for training and 1 for testing, and repeat the whole experiments ten times (thus 100 runs per dataset) and report the average prediction accuracies and standard deviations. The ranges of hyperparameters $\gamma$ and $D\_max$ are [1e-3 1e-2 1e-1 1 10] and [3:3:30], respectively. All parameters of the SVM and hyperparameters of our method were optimized only on the training dataset. 
The node embedding size is set to either 4, 6 or 8 but always be the same number for all variants of RGE on the same datasets. 
To eliminate the random effects, we repeat the whole experiments ten times and report the average prediction accuracies and standard deviations. 
For all baselines we take the best number reported in the papers except EMD, where we rerun the experiments for fair comparisons in terms of both accuracy and runtime. Since GRE, OA, EMD, and PM are essentially built on the same node embeddings from the adjacency matrices, we take the number of OA-{$E_\lambda$}(A) in \cite{johansson2015learning} for a fair comparison.

\end{document}